\documentclass{article}
\usepackage[utf8]{inputenc}
\usepackage{hyperref}
\usepackage{amsmath,amsfonts,amssymb,amsthm,bm,graphicx,color,cleveref,url}
\usepackage{mathrsfs}
\usepackage[title]{appendix}

\newcommand{\RR}{\mathbb{R}}
\newcommand{\EE}{\mathbb{E}}
\newcommand{\cdt}{\operatorname{LOT}}

\newcommand{\id}{\operatorname{Id}}

\newcommand{\E}{\mathcal{E}}

\newcommand{\HH}{\mathcal{H}}
\newcommand{\GG}{\mathcal{G}}

\newcommand{\BB}{\mathscr{P}_2}
\newcommand{\FF}{L^2(\RR^n,\mfix)}
\newcommand{\FFm}{L^2(\RR^n,\mfixAdd)}
\newcommand{\mfix}{\sigma}
\newcommand{\mfixAdd}{\mu}
\newcommand{\mvariable}{\nu}

\newcommand{\supp}{\operatorname{supp}}
\newcommand{\conv}{\operatorname{conv}}

\definecolor{darkgreen}{rgb}{0, 0.5,0}

\newtheorem{definition}{Definition}[section]
\newtheorem{remark}[definition]{Remark}
\newtheorem{lemma}[definition]{Lemma}
\newtheorem{theorem}[definition]{Theorem}
\newtheorem{corollary}[definition]{Corollary}
\newtheorem*{theorem*}{Theorem}

\usepackage{color}

\title{Linear Optimal Transport Embedding: Provable Wasserstein classification for certain rigid transformations and perturbations}
\author{Caroline Moosm\"uller\thanks{Department of Mathematics, University of California, San Diego, CA
  (cmoosmueller@ucsd.edu, acloninger@ucsd.edu).}
\and Alexander Cloninger\footnotemark[1] \thanks{Halicio{\u g}lu Data Science Institute, University of California, San Diego, CA}}
\date{}

\begin{document}

\maketitle

\begin{abstract}
{Discriminating between distributions is an important problem in a number of scientific fields. This motivated the introduction of Linear Optimal Transportation (LOT), which embeds the space of distributions into an $L^2$-space. The transform is defined by computing the optimal transport of each distribution to a fixed reference distribution, and has a number of benefits when it comes to speed of computation and to determining classification boundaries.
In this paper, we characterize a number of settings in which  LOT embeds families of distributions into a space in which they are linearly separable.  This is true in arbitrary dimension, and for families of distributions generated through perturbations of shifts and scalings of a fixed distribution.  
We also prove conditions under which the $L^2$ distance of the LOT embedding between two distributions in arbitrary dimension is nearly isometric to Wasserstein-2 distance between those distributions.  This is of significant computational benefit, as one must only compute $N$ optimal transport maps to define the $N^2$ pairwise distances between $N$ distributions.  We demonstrate the benefits of LOT on a number of distribution classification problems.}\\
{Keywords: Optimal transport, linear embedding, Wasserstein distance, classification}
\\
2000 Math Subject Classification: 60D05, 68T10, 68T05
\end{abstract}

%%%%%%%%%%%%%%%%%%%%%%%%%%%%%%%%%%%%%%%%%%%%%%%%%%%%%%%%%%%%%%%%%%%%%%%%%%%%%
%%%%%%%%%%%%%%%%%%%%%%%%%%%%%%%%%%%%%%%%%%%%%%%%%%%%%%%%%%%%%%%%%%%%%%%%%%%%%%
\section{Introduction}
\label{sec:intro}
%%%%%%%%%%%%%%%%%%%%%%%%%%%%%%%%%%%%%%%%%%%%%%%%%%%%%%%%%%%%%%%%%%%%%%%%%%%%%
%%%%%%%%%%%%%%%%%%%%%%%%%%%%%%%%%%%%%%%%%%%%%%%%%%%%%%%%%%%%%%%%%%%%%%%%%%%%

The problem of supervised learning is most commonly formulated as follows.  Given data of the form $\{(x_i, y_i)\}_{i=1}^N$ where $x_i \in \RR^n$, learn a function $f:\RR^n\rightarrow \RR$ such that $f(x_i)\approx y_i$.  However, in many applications the data points are not simply points in $\RR^n$, but are instead probability measures $\mu_i$ on $\RR^n$, or even finite samples $X_i = \{x_j^{(i)}\}_{j=1}^{N_i}$ for $x_j^{(i)}\sim \mu_i$.  Applications where this problem arises are surveys broken into demographic or location groups \cite{cloninger2019people}, topic modeling from a bag of words model \cite{zhang2010understanding}, and flow cytometry and other measurements of cell or gene populations per person \cite{bruggner2014automated, cheng2017two, zhao2020detecting}.

The most natural way to solve the supervised learning problem on data $\{(\mu_i, y_i)\}_{i=1}^N$ is to embed $\mu_i$ into a (possibility infinite dimensional) Euclidean space and then apply traditional machine learning techniques on this embedding.  Simple versions of this embedding would be through moments 
$
\mu_i \mapsto \EE_{X\sim \mu_i}[X]
$
\cite{newey1987hypothesis}, or through a mean embedding
$
\mu_i \mapsto \EE_{X\sim \mu_i} K(\cdot, X)
$
for some kernel $K$ \cite{muandet2016kernel}.  However, these embeddings either throw away pertinent information about $\mu_i$ (e.g., higher order moments), or induce a complex nonlinear geometric relationship between distributions (e.g., $\|\EE_{X\sim \mu(x)} K(\cdot, X) - \EE_{X\sim \mu(x-\tau)} K(\cdot, X)\| \approx \|\EE_{X\sim \mu(x)} K(\cdot, X) - \EE_{X\sim \mu(x-2\tau)} K(\cdot, X)\|$ for $\tau$ significantly larger than the bandwidth of the kernel).    These issues motivate the need for a transformation that is both injective and  induces a simple geometric structure in the embedding space, so that one can learn an easy classifier.

The natural distance between distributions is Wasserstein-2 distance \cite{villani-2009}, where the distance between distributions $\mu$ and $\nu$ is 
\begin{align}\label{eq:wasserstein}
W_2(\mu,\nu)^2 = \min_{T\in \Pi_\mu^\nu} \int \|T(x)- x\|^2 d\mu(x),
\end{align}
where $\Pi_\mu^\nu$ is the collection of all measure preserving maps from $\mu$ to $\nu$.  The $\arg\min$ of \eqref{eq:wasserstein} is referred to as the ``optimal transport map'' and we denote it by $T_\mu^\nu$ (see Section \ref{sec:OMT} for a full description).   
Wasserstein distance is a more natural distance between distributions as it is a metric on distributions (unlike distances between a finite number of moments as above) and the distance does not saturate as the distributions move further apart (unlike mean embeddings as described above).
Optimal transport has been of significant importance in machine learning, including as a cost for generative models \cite{arjovsky2017wasserstein}, natural distances between images \cite{rubner2000earth}, pattern discovery for data cubes of neuronal data \cite{mishne2016hierarchical}, and general semi-supervised learning \cite{solomon2014wasserstein}.  There are two main drawbacks to optimal transport in machine learning.  The first is that the computation of each transport map is slow, though this has motivated a number of approximations for computational speed up \cite{cuturi2013sinkhorn, shirdhonkar2008approximate, leeb2016holder}.  The second drawback is that it is difficult to incorporate supervised learning into optimal transport, as the distance is defined for a pre-defined cost function and \cref{eq:wasserstein}, as stated, does not generate a feature embedding of $\mu$ and $\nu$ that can be fed into traditional machine learning techniques.

This motivated the introduction of Linear Optimal Transportation (LOT) \cite{wei13}, also called Monge embedding in \cite{merigot20}.
LOT is a set of transformations based on optimal transport maps, which map a distributions $\mu$ to the optimal transport map that takes a fixed reference distribution $\sigma$ to $\mu$:
\begin{align}\label{intro:lot}
\mu \mapsto T_\sigma^{\mu}.
\end{align}
The power of this transform lies in the fact that the nonlinear space of distributions is mapped into the linear space of $L^2$ functions. In addition, \cref{intro:lot} is an embedding with convex image.

In 1D, the optimal transport map is simply the generalized cdf of the distribution (if $\sigma=\text{Unif}([0,1])$ this is exactly the traditional cdf).  In \cite{park18}, the authors define the LOT as the Cumulative Distribution Transform (CDT), and the main theory and applications presented in \cite{park18} concern linear separability of data consisting of 1-dimensional densities.

However, LOT is more complicated on $\RR^n$ for $n>1$. For $n=1$, the cdf is the only measure preserving map from $\mu_i$ to $\sigma$, and it can be computed explicitly. This is not the case for $n>1$: There are a large number of measure preserving maps, with the optimal transport map being the map that requires minimal work, see \eqref{eq:wasserstein}.   Similarly, there are a much larger family of potential simple continuous perturbations that can be done to $\mu_i$ when $n>1$ (e.g., sheerings, rotations) than exist for $n=1$.  

In \cite{kolouri-2016}, the CDT is combined with the Radon transform to apply results from \cite{park18} in general dimensions $n>1$. While this construction can be considered a variant of LOT, a linear separability result for LOT in $n>1$ is still missing. A proof of linear separability in LOT space for $n>1$ is one of the main contributions
of this paper (see \Cref{sec:main_contr}).

The LOT embedding \cref{intro:lot} comes with yet another advantage. One can define a distance between two distributions
$\mu_i$ and $\mu_j$ as the $L^2$-norm of their images under LOT:
\begin{align*}
W_2^{\cdt}(\mu_i, \mu_j)^2 &:= \|T_\sigma^{\mu_i} - T_\sigma^{\mu_j}\|_{\mfix}^2=\int \|T_\sigma^{\mu_i}(x) - T_\sigma^{\mu_j}(x)\|^2 \, d\sigma(x).
\end{align*}
In this paper, we prove that $W_2$ equals $W_2^{\cdt}$ if the family of distributions $\mfixAdd_i$ is generated by shifts and scalings of a fixed distribution $\mfixAdd$. We further show that $W_2$ is well approximated by $W_2^{\cdt}$ for perturbations of shift and scalings (see \Cref{sec:main_contr}).

We wish to highlight the computational importance of establishing approximate equivalence between LOT distance and Wasserstein-2 distance.  Given $N$ distributions, computing the exact Wasserstein-2 distance between all distributions naively requires computing ${N \choose 2}$ expensive OT optimization problems.  However, if the distributions come from a family of distributions generated by perturbations of shifts and scalings, one can instead compute $N$ expensive OT optimization problems mapping each distribution to $\sigma$ and compute ${N \choose 2}$ cheap Euclidean distances between the transport maps, and this provably well approximates the ground truth distance matrix.
%In addition, the LOT embedding can be combined with algorithms such as Sinkhorn \cite{cuturi2013sinkhorn} to further improve the computational speed.

%%%%%%%%%%%%%%%%%%%%%%%%%%%%%%%%%%%%%%%%
\subsection{Main contributions}
\label{sec:main_contr}
%%%%%%%%%%%%%%%%%%%%%%%%%%%%%%%%%%%%%
The main contributions of this paper are as follows:
\begin{itemize}
\item We establish the following with regards to building simple classifiers:
\begin{theorem}[Informal Statement of Theorem \ref{thm:almost_separability}]
If $\mathcal{P}=\{\mu_i : y_i=1\}$ are $\varepsilon$-perturbations of shifts and scalings of $\mu$, and $\mathcal{Q}=\{\nu_i : y_i=-1\}$ are $\varepsilon$-perturbations of shifts and scalings of $\nu$, and $\mathcal{P}$ and $\mathcal{Q}$ have a small minimal distance depending on $\varepsilon$ (and satsify a few technical assumptions), then $\mathcal{P}$ and $\mathcal{Q}$ are linearly separable in the LOT embedding space.
\end{theorem}
\item 
 We establish the following with regards to LOT distance:
 \begin{theorem}[Informal Statement of Theorem \ref{thm:CDT_perturb}]
 If $\mu$ and $\nu$ are $\varepsilon$-perturbations by shifts and scalings of one another, then 
 \begin{align*}
 W_2(\mu,\nu)\le W_2^{\cdt}(\mu,\nu)\le W_2(\mu,\nu) + C_{\sigma} \varepsilon + \overline{C_{\sigma}} \varepsilon^{1/2}.
 \end{align*}
 \end{theorem}
 In particular, this implies that the LOT embedding is an isometry on the subset of measures related via shifts and scalings, i.e.\ when $\varepsilon=0$.
\item We demonstrate that in applications to MNIST images, the LOT embedding space is near perfectly linearly separable between classes of images.
\end{itemize}

%%%%%%%%%%%%%%%%%%%%%%%%%%%%%%%%%%%%%%%%%%%%%%%%%%%%%%%%%%%%%
%%%%%%%%%%%%%%%%%%%%%%%%%%%%%%%%%%%%%%%%%%%%%%%%%%%%%%%%%%%%%
\section{Preliminaries: Optimal Mass Transport}
\label{sec:OMT}
%%%%%%%%%%%%%%%%%%%%%%%%%%%%%%%%%%%%%%%%%%%%%%%%%%%%%%%%%%%%%
%%%%%%%%%%%%%%%%%%%%%%%%%%%%%%%%%%%%%%%%%%%%%%%%%%%%%%%%%%%%%
%
%
Let $\mathscr{P}(\RR^n)$ be the set of probability measures on $\mathbb{R}^n$.
By $\BB(\RR^n)$ we denote those measures in $\mathscr{P}(\RR^n)$ with bounded second moment, i.e.\ $\mfix \in \mathscr{P}(\RR^n)$ that satisfy
\begin{equation*}
    \int \|x\|_2^2\, d\mfix(x) < \infty.
\end{equation*}
For $\mfix \in \BB(\mathbb{R}^n)$ we also consider the space $\FF$ with norm
\begin{equation*}
\|f\|_{\mfix}^2 = \int \|f(x)\|_2^2\, d\mfix(x).    
\end{equation*}
In case of the $L^2$-norm with respect to the Lebesgue measure $\lambda$, we simply write $\|f\|$.

For a map $T:\RR^n \to \RR^n$ and a measure $\mfix$ we define the \emph{pushforward measure} $T_{\sharp}\mfix$ by 
\begin{equation*}
    T_{\sharp}\mfix(A) = \mfix(T^{-1}(A)),
\end{equation*}
where $A\subset \RR^n$ is measurable and $T^{-1}(A)$ denotes the preimage of $A$ under $T$.

If $\mfix \in \BB(\RR^n)$ is absolutely continuous with respect to the Lebesgue measure $\lambda$, which we denote by $\mfix \ll \lambda$, then there exists a density $f_{\mfix}: \RR^n \to \RR$ such that
\begin{equation}\label{eq:densities}
    \mfix(A) = \int_A f_{\mfix}(x)\,d\lambda(x), \quad A \subseteq \RR^n \text{ measureable.}
\end{equation}
In terms of densities, the pushforward relation $\mvariable(A) = \mfix(T^{-1}(A))$ is given by
\begin{equation}\label{eq:measure_preserving}
    \int_{T^{-1}(A)} f_{\mfix}(x)\,d\lambda(x) = \int_A f_{\mvariable}(y)\,d\lambda(y), \quad A\subseteq \RR^n \text{ measurable.}
\end{equation}
In case the map $T$ is invertible and differentiable, we can rewrite \eqref{eq:measure_preserving} as 
\begin{equation}\label{eq:bijective_transport}
    f_{\mvariable}(y)=f_{\mfix}\left(T^{-1}(y)\right)|\det D_yT^{-1}|.
\end{equation}
Given two measures ${\mfix},{\mvariable}$, there can exist many maps $T$ that push $\mfix$ to $\mvariable$. Therefore one seeks to impose yet another condition to make this map unique. In the theory of optimal transport \cite{villani-2009}, uniqueness is obtained via solving an optimization problem. The map $T$ is required to minimize a cost function of the form
\begin{equation}\label{eq:monge}
    \int c(T(x),x) \,d\mfix(x),
\end{equation}
under the constraint $T_{\sharp}\mfix = \mvariable$.
In this paper we consider the cost $c(x,y) = ||x-y||_2^2$. Other cost functions are possible as well, most notably, $p$-norms can be studied instead of $2$-norms \cite{villani-2009}.
If the optimization has a solution, then
\begin{equation*}
    W_2({\mfix},{\mvariable})^2 = \min_{T: T_{\sharp}\mfix = \mvariable} 
    \int ||T(x)-x||_2^2 \,d\mfix(x)
\end{equation*}
is the \emph{$2$-Wasserstein distance} between the measures ${\mfix}$ and ${\mvariable}$. In this paper we will refer to $W_2$ as \emph{the} Wasserstein distance, as we only consider this case. The map $T$ that minimizes \eqref{eq:monge} is called \emph{optimal transport map}.

We introduce the notation $T_{\mfix}^{\mvariable}$ to denote the optimal transport map from $\mfix$ to $\mvariable$. With this notation we have the identity
\begin{equation*}
    W_2(\mfix,\mvariable) = \|T_{\mfix}^{\mvariable}-\id\|_{\mfix}.
\end{equation*}
We now cite a result concerning existence and uniqueness of the optimal transport map which is used throughout this paper.

%
%%%%%%%%%%%%%%%%%%%%%%%%%%%%%%%%%%%%%%%%%%%%%%%%%%%%%%%%%%%%%%%%%%%%%%%%%%
\begin{theorem}[\cite{brenier-1991}, formulation taken from \cite{peyre19}]
\label{thm:brenier}
%%%%%%%%%%%%%%%%%%%%%%%%%%%%%%%%%%%%%%%%%%%%%%%%%%%%%%%%%%%%%%%%%%%%%%%%%%
%
Let $\mfix,\mvariable \in \BB(\RR^n)$ and consider the cost function $c(x,y)=||x-y||_2^2$.
If $\mfix$ is absolutely continuous with respect to the Lebesgue measure, then there exists a unique map $T \in \FF$ pushing $\mfix$ to $\mvariable$, which minimizes \eqref{eq:monge}. Furthermore, the map $T$ is uniquely defined as the gradient of a convex function $\varphi$,
$T (x) = \nabla \varphi(x)$, where $\varphi$ is the unique (up to an additive constant) convex function
such that $(\nabla \varphi)_{\sharp}\mfix = \mvariable$. 
\end{theorem}
There exist many generalizations of this result, for example to more general cost functions, or to Riemannian manifolds \cite{brenier-1991,mccann-2001,villani-2009,ambrosio-2013}.

%%%%%%%%%%%%%%%%%%%%%%%%%%%%%%%%%%%%%%%%%%%%%%%%%%%%%%%%%%%%%%%%%
%%%%%%%%%%%%%%%%%%%%%%%%%%%%%%%%%%%%%%%%%%%%%%%%%%%%%%%%%%%%%%%%%
\section{Linear Optimal Transportation and its properties}
\label{sec:IntroLOT}
%%%%%%%%%%%%%%%%%%%%%%%%%%%%%%%%%%%%%%%%%%%%%%%%%%%%%%%%%%%%%%%%%
%%%%%%%%%%%%%%%%%%%%%%%%%%%%%%%%%%%%%%%%%%%%%%%%%%%%%%%%%%%%%%%%%

In this section we introduce the \emph{Linear Optimal Transportation (LOT)} as defined in \cite{wei13} (also called \emph{Monge embedding} in \cite{merigot20}) and present its basic properties.

LOT is an embedding of $\BB(\RR^n)$ into the linear space $\FF$ based on a fixed measure $\mfix$. It is defined as
\begin{equation}\label{eq:LOT}
    \mvariable \mapsto T_{\mfix}^{\mvariable}.
\end{equation}
The power of this embedding lies in the fact that the target space is linear. This allows to apply linear methods to inherently nonlinear problems in $\BB(\RR^n)$ (see, for example, the application to classification problems in \Cref{sec:examples}, and \cite{park18,merigot20}).

The map \eqref{eq:LOT} can be thought of as a linearization of the Riemannian manifold $\BB(\mathbb{R}^n)$ endowed with the Wasserstein distance. The tangent space of $\BB(\mathbb{R}^n)$ at $\sigma$ lies in $\FF$, hence \eqref{eq:LOT} is an inverse to the exponential map \cite{villani-2009,gigli-2011,wei13}.

The map \cref{eq:LOT} has been studied by others authors as well, mainly with respect to its regularity. \cite{gigli-2011} shows $1/2$-H\"older regularity of a time-dependent version of \cref{eq:LOT} under regularity assumptions on the measures $\mfix,\mvariable$ (we discuss this result in Appendix \ref{sec:regularity}). \cite{merigot20} prove a weaker H\"older bound, but without any regularity assumptions on the measures. It is also shown in both \cite{gigli-2011} and \cite{merigot20} that in general, the regularity of \eqref{eq:LOT} is not better than $1/2$.

Bounds for a variant of \eqref{eq:LOT} in which the source measure, rather than the target measure is varied, can be found in \cite{berman20}.

We now define LOT and summarize its basic properties.

%
%
%%%%%%%%%%%%%%%%%%%%%%%%%%%%%%%%%%%%%%%%%%%
\begin{definition}[Linear Optimal Transportation \cite{wei13}]\label{def:cdt}
%%%%%%%%%%%%%%%%%%%%%%%%%%%%%%%%%%%%%%%
Fix a measure $\mfix \in \BB(\RR^n)$, $\mfix \ll \lambda$. We define the \emph{Linear Optimal Transportation (LOT)}, $F_{\mfix}$, which assigns a function in $\FF$ to a measure in $\BB(\RR^n)$:
\begin{equation*}
    F_{\mfix}(\mvariable) = T_{\mfix}^{\mvariable},
    \quad \mvariable \in \BB(\RR^n).
\end{equation*}
%
%
%%%%%%%%%%%%%%%%%%%%%%%%%%%%%%%%%%%%%%%%%%%%%%%%%%%%%%%
\end{definition}
%%%%%%%%%%%%%%%%%%%%%%%%%%%%%%%%%%%%%%%%%%%%%%%%%%%%%%%
% %
%
We now show that LOT is an embedding with convex image:
%
%%%%%%%%%%%%%%%%%%%%%%%%%%%%%%%%%%%%%%%%%%%%%%%%%%%%%%%%%
\begin{lemma}\label{lem:injective}
%%%%%%%%%%%%%%%%%%%%%%%%%%%%%%%%%%%%%%%%%%%%%%%%%%%%%%%%%
For fixed $\mfix \in \BB(\RR^n)$, $\mfix \ll \lambda$, we have the following
\begin{enumerate}
    \item $F_{\mfix}$ embeds $\BB(\RR^n)$ into $\FF$;
%    \item $F_{\mfix}$ is invertible on its image; and
    %
    \item the image $F_{\mfix}(\BB(\RR^n))$ is convex in $\FF$.
\end{enumerate}
%

%%%%%%%%%%%%%%%%%%%%%%%%%%%%%%%%%%%%%%%%%%%%%%%%%%%%%%%%%
\end{lemma}
%%%%%%%%%%%%%%%%%%%%%%%%%%%%%%%%%%%%%%%%%%%%%%%%%%%%%%%%%
\begin{proof}
%%%%%%%%%%%%%%%%%%%%%%%%%%%%%%%%%%%%%%%%%%%%%%%%%%%%%%%%%
The proof is an application of \Cref{thm:brenier}. The first part is also shown in \cite{merigot20}. For the convenience of the reader, we summarize the proof in \Cref{sec:ProofsIntroLOT}.
\end{proof}
%%%%%%%%%%%%%%%%%%%%%%%%%%%%%%%%%%%%%%%%%%%%%%%%%%%%%%%%%
%
%
%
%
We introduce a compatibility condition between LOT and the pushforward operator, which is one of the key ingredients for the results in \Cref{sec:linear_sep}. 

Fix two measures $\mfix,\mfixAdd \in \BB(\RR^n), \mfix \ll \lambda$. $F_{\mfix}$ is called \emph{compatible} with $\mfixAdd$-pushforwards of a set of functions $\HH \subseteq \FFm$ if for every $h\in \HH$ we have
\begin{equation}\label{eq:compatible}
    F_{\mfix}({h}_{\sharp}\mfixAdd) = h \circ F_{\mfix}(\mfixAdd).
\end{equation}
This condition has also been introduced by \cite{aldroubi20} on the level of densities. 
\begin{remark}
For $\mfix = \mfixAdd$, the compatibility condition reads as $T_{\mfix}^{h_{\sharp}\mfix} = h$. This means that a function $h$ is required to be the optimal transport from $\mfix$ to $h_{\sharp}\mfix$. This is a rather strong condition, and not satisfied for a general function $h$. 

The compatibility condition can also be understood in terms of operators.
The pushforward operator $h \mapsto h_{\sharp}\mfix$, which in Riemannian geometry is an exponential map, is left-inverse to $F_{\mfix}$. The compatibility condition requires that it is also right-inverse.

We mention below that the compatibility condition is satisfied for shifts and scalings, a fact also shown in \cite{aldroubi20} on the level of densities. \cite{aldroubi20} also prove that shifts and scalings are the only transformation that satisfy \eqref{eq:compatible}.
\end{remark}
For $a\in \RR^n$ denote by $S_a(x)=a+x$ the shift by $a$. Similarly, for $c>0$ denote by $R_c(x)=cx$ the scaling by $c$. We denote by $\E:=\{S_a: a\in \RR^n\} \cup \{R_c: c>0\}$ the set of all shifts and scalings.
%
%
%%%%%%%%%%%%%%%%%%%%%%%%%%%%%%%%%%%%%%%%%%%%%%%%%%%%%%%%%%%%%%
\begin{lemma}[Compatibility on $\RR$ and with shifts and scalings]
\label{lem:translate_transport}
%%%%%%%%%%%%%%%%%%%%%%%%%%%%%%%%%%%%%%%%%%%%%%%%%%%%%%%%%%%%%%%
%

Let $\mfix,\mfixAdd \in \BB(\RR^n)$, $\mfix \ll \lambda$.
\begin{enumerate}
    \item If $n=1$, i.e.\ on $\RR$, $F_{\mfix}$ is compatible with $\mfixAdd$-pushforwards of monotonically increasing functions.
    \item For general $n\geq 1$, $F_{\mfix}$ is compatible with $\mfixAdd$-pushforwards of shifts and scalings, i.e. functions in $\E$.
\end{enumerate}
%%%%%%%%%%%%%%%%%%%%%%%%%%%%%%%%%%%%%%%%%%%%%%%%%%%%%%%%%%%%%%%%%%%%%%%%
\end{lemma}
%%%%%%%%%%%%%%%%%%%%%%%%%%%%%%%%%%%%%%%%%%%%%%%%%%%%%%%%%%%%%%%%%%%%%%%%
\begin{proof}
%%%%%%%%%%%%%%%%%%%%%%%%%%%%%%%%%%%%%%%%%%%%%%%%%%%%%%%%%%%%%%%%%%%%%%
The proof is an application of \Cref{thm:brenier} and can be found in \cite{aldroubi20} (the first part can also be found in \cite{park18}); for the convenience of the reader, we show details in \Cref{sec:ProofsIntroLOT}.
%%%%%%%%%%%%%%%%%%%%%%%%%%%%%%%%%%%%%%%%%%%%%%%%
\end{proof}
%%%%%%%%%%%%%%%%%%%%%%%%%%%%%%%%%%%%%%%%%%%%%%%%%

%%%%%%%%%%%%%%%%%%%%%%%%%%%%%%%%%%%%%%%%%%%%%%%%%%%%%%%%%%%%%%%%%
%%%%%%%%%%%%%%%%%%%%%%%%%%%%%%%%%%%%%%%%%%%%%%%%%%%%%%%%%%%%%%%%%
\section{Geometry of LOT embedding space}
\label{sec:linear_sep}
%%%%%%%%%%%%%%%%%%%%%%%%%%%%%%%%%%%%%%%%%%%%%%%%%%%%%%%%%%%%%%%%%
In this section, we characterize the geometry of the LOT embedding space under families of compatible transformations in $\E$ (i.e., shifts and scalings), as well as for approximately compatible transformations in $\GG_{\lambda,R,\varepsilon}$ (\cref{eq:G}), where $\lambda$ denotes the Lebesgue measure.

Recall from \Cref{lem:isometry_compatible} that for a measure $\mfixAdd$ and a set of functions $\HH$ we denote by $\HH \star \mfixAdd$ the set of all pushforwards of $\mfixAdd$ under $\HH$, i.e.\ 
\begin{equation*}
    \HH \star \mfixAdd = \{h_{\sharp}\mfixAdd: h\in \HH \}.
\end{equation*}
In this section, we are mainly interested in conditions under which two families of distributions defined by pushforwards of $\GG \subset \GG_{\lambda,R,\varepsilon}$, $\GG \star \mfixAdd$ and $\GG \star \mvariable$, are linearly separable in the LOT embedding space. 

Before stating the main results of this section, we briefly describe linear separability and its importance in machine learning.  Linear separability of two disjoint sets in a Hilbert space implies the existence of a hyperplane $w(x)=b$ such that 
\begin{align*}
\langle w, \mu_i\rangle < b, & & \forall \mu_i \in \HH \star \mfixAdd\\
\langle w, \nu_i\rangle > b, & & \forall \nu_i \in \HH \star \mvariable.
\end{align*}
The existence of such a hyperplane can be established through the Hahn-Banach separation theorem.  The theorem simply assumes that the two sets ($\HH \star \mfixAdd$, $ \HH \star \mvariable$) are convex, and that one is closed and the other is compact \cite{narici2010topological}.   
 
Linear separability is a strong and important condition for many machine learning applications and supervised learning generally. This is because learning a linear classifier is very straightforward, and does not require many training points to accurately estimate $w(x)$. This implies that once the distributions are mapped to the LOT embedding space, it is possible to learn a classifier that perfectly separates the two families with only a small amount of labeled examples.

We note that the result on $\GG_{\lambda,\varepsilon,R}$ (\Cref{thm:almost_separability}) is the main result of this section, but we list several other results for completeness.   We also note that, for ease of understanding, we frame all theorems in this section for subsets of shifts/scalings or perturbations of such.   However, Theorems \ref{cor:linear_sep_trans_scale} and \ref{thm:almost_separability} actually have versions in Appendix \ref{appendix:lin_sep} (\Cref{cor:linear_sep,cor:linear_sep_eps}, respectively) for the family of all approximately compatible transformations.
Furthermore, in the case of $\GG_{\lambda,\varepsilon,R}$ (\Cref{thm:almost_separability}), through \Cref{cor:push_CDT_bound_L2}, we can give an explicit characterization of the minimal distance $\delta$ required between the two families of distributions, $\GG \star \mfixAdd$ and $\GG \star \mvariable$, to guarantee linear separability.
%%%%%%%%%%%%%%%%%%%%%%%%%%%%%%%%%%%%%%%%%%%%%%%%%%%%%%%%%%%%%%%%

\subsection{Approximation of the Wasserstein distance}

From \Cref{lem:injective} we know that LOT embeds $\BB(\RR^n)$ into $\FF$. In general, this embedding is not an isometry.

In this section we derive the error that occurs when approximating the Wasserstein distance by the $L^2$ distance obtained in the LOT embedding.
We are thus interested in the accuracy of the following approximation:
\begin{equation}\label{eq:W2approxLOT}
    W_2(\mfixAdd,\mvariable) \approx \|F_{\mfix}(\mfixAdd)-F_{\mfix}(\mvariable)\|_{\mfix}.
\end{equation}
Note that if $W_2(\mfixAdd,\mvariable)$ is approximated well by $\|F_{\mfix}(\mfixAdd)-F_{\mfix}(\mvariable)\|_{\mfix}$, LOT is very powerful, as the Wasserstein distance between $k$ different measures can be computed from only $k$ transports instead of ${k \choose 2}$.
Indeed, in this section we show that \cref{eq:W2approxLOT} is exact, i.e.\ the LOT embedding is an isometry, for two important cases: On $\RR$, and on $\RR^n$ if both $\mfixAdd$ and $\mvariable$ are pushforwards of a fixed measure under shifts and scalings. We further show that it is almost exact for pushforwards of functions close to shifts and scalings. 
%
%%%%%%%%%%%%%%%%%%%%%%%%%%%%%%%%%%%%%%%%%%%%%%%%%%%%%%%%%%%%%%%%%%%%%%%%%%%%%%%%%%%

It is important to note that in most applications, distributions are not exact shifts or scalings of one another.  There always exist perturbations of distributions, whether it is rotation, stretching, or sheering.  Thus, it is important to consider the behavior of LOT under such perturbations, and demonstrate that the LOT distance continues to be a quasi-isometry with respect to Wasserstein-2 distance and that the deformation constants depend smoothly on the size of the pertubation.

Let $\mfixAdd \in \BB(\RR^n)$, $R>0$, and $\varepsilon>0$. Recall that we denote by $\E=\{S_a: a\in \RR^n\} \cup \{R_c: c>0\}$ the set of all shifts and scalings. We define the sets
\begin{equation}\label{eq:E}
    \E_{\mfixAdd,R}=\{h \in \E: \|h\|_{\mfixAdd}\leq R\}
\end{equation}
and
\begin{equation}\label{eq:G}
 \GG_{\mfixAdd,R,\varepsilon} = \{g\in L^2(\RR^n,\mfixAdd): \exists h \in \E_{\mfixAdd,R}: \|g-h\|_{\mfixAdd}\leq \varepsilon\}.
\end{equation}
This can be thought of as the $\varepsilon$ tube around set of shifts and scalings, or as the set of almost compatible transformations.
%
%
%%%%%%%%%%%%%%%%%%%
\begin{theorem}
\label{thm:CDT_perturb}
%%%%%%%%%%%%%%%%%%%
%{\color{red} merge with 2/15}
Let $\mfix,\mfixAdd \in \BB(\RR^n), \mfix,\mfixAdd \ll \lambda$. Let $R>0,\varepsilon>0$.
\begin{enumerate}
\item %\label{it:2/15}
For $g_1,g_2 \in \GG_{\mfixAdd,R,\varepsilon}$ and $\mfix$ the Lebesgue measure on a convex, compact subset of $\RR^n$, we have
\begin{equation*}
    0\leq \|F_{\mfix}({g_1}_{\sharp}\mfixAdd)-F_{\mfix}({g_2}_{\sharp}\mfixAdd)\|_{\mfix}-W_2({g_1}_{\sharp}\mfixAdd,{g_2}_{\sharp}\mfixAdd) \leq C\varepsilon^{\frac{2}{15}}+2\varepsilon.
\end{equation*}

\item %\label{it:1/2}

If $\mfix,\mfixAdd$ satisfy the assumptions of Caffarelli's regularity theorem (\Cref{thm:caffarelli-regularity}),
then for $g_1,g_2 \in \GG_{\mfixAdd,R,\varepsilon}$ we have
\begin{equation*}
    0\leq \|F_{\mfix}({g_1}_{\sharp}\mfixAdd)-F_{\mfix}({g_2}_{\sharp}\mfixAdd)\|_{\mfix}-W_2({g_1}_{\sharp}\mfixAdd,{g_2}_{\sharp}\mfixAdd) \leq \overline{C}\,\varepsilon^{1/2}+C\,\varepsilon .
\end{equation*}

\end{enumerate}
The constants depend on $\mfix,\mfixAdd$ and $R$.

%%%%%%%%%%%%%%%%%%%
\end{theorem}
%%%%%%%%%%%%%%%%%%%
%%%%%%%
\begin{proof}
%%%%%%
The main ingredient for these results are H\"older bounds as derived in \cite{merigot20,gigli-2011}. We show a detailed proof in \Cref{sec:ProofsLinearSep}.
%%%%%%
\end{proof}

We mention that through the application of results derived from \cite{gigli-2011} (\Cref{cor:push_CDT_bound_L2}), the constants appearing in the second part of this theorem can be characterized explicitly, see \Cref{sec:ProofsLinearSep}.

The theorem states that for functions close to ``ideal'' functions (shifts and scalings), the LOT embedding is an almost isometry. Also note the trade-off between H\"older regularity and regularity assumptions on $\mfix,\mfixAdd$: Through \cite{merigot20}, we can achieve a $2/15$ bound without strong regularity assumptions on $\mfix,\mfixAdd$; the bound improves through \cite{gigli-2011}, when $\mfix,\mfixAdd$ are regular in the sense of \Cref{thm:caffarelli-regularity}.

Without perturbation, i.e.\ when $\varepsilon=0$, \Cref{thm:CDT_perturb} implies
%%
%%%%%%%%%%%%%%%%%%%%%%%%%%%%%%%%%%%%%
\begin{corollary}
\label{lem:isometry_compatible}
%%%%%%%%%%%%%%%%%%%%%%%%%%%%%%%%%%%%
Let $\mfix, \mfixAdd \in \BB(\RR^n)$, $\mfix \ll \lambda$. Then for $h_1,h_2 \in \mathcal{E}$ we have
\begin{equation*}
    W_2({h_1}_{\sharp}\mfixAdd,{h_2}_{\sharp}\mfixAdd) = 
    \|F_{\mfix}({h_1}_{\sharp}\mfixAdd) -
    F_{\mfix}({h_2}_{\sharp}\mfixAdd)\|_{\mfix} = \|h_1-h_2\|_{\mfixAdd}.
\end{equation*}
% 
%%%%%%%%%%%%%%%%%%%%%%%%%%%%%%%%%%%%%%%%
\end{corollary}
%%%%%%%%%%%%%%%%%%%%%%%%%%%%%%%%%%%%%%%%
This means that $F_{\mfix}$ restricted to $\mathcal{E} \star \mfixAdd:= \left\{h_{\sharp}\mfixAdd: h\in \mathcal{E} \right\}$ is an isometry.

%%%%%%%%%%%%%%%%%%%%%%%%%%%
We also have the following result, which as also been shown in \cite{park18}:
%%%%%%%%%%%%%%%%%%%%%%%%%%%%%%%%%%%%%%%%%
\begin{corollary}
\label{cor:isometryR}
%%%%%%%%%%%%%%%%%%%%%%%%%%%%%%%%%%%%
On $\RR$, $F_{\mfix}$ is an isometry.
%%%%%%%%%%%%%%%%%%%%%%%%%%%%%%%%%%%%%%
\end{corollary}
%%%%%%%%%%%%%%%%%%%%%%%%%%%%%%%%%%%%%%
\begin{proof}
%%%%%%%%%%%%%%%%%%%%%%%%%%%%%%%%%%%%%%
We proof in \Cref{lem:compatible_isometry} that compatibility of $F_{\mfix}$ with $\mu$-pushforwards implies \cref{eq:transport_composition}. Thus the result follows from \Cref{lem:translate_transport}.
%%%%%%%%%%%%%%%%%%%%%%%%%%%%%%%%%%%%%
\end{proof}
%%%%%%%%%%%%%%%%%%%%%%%%%%%%%%%%%%%%%

\subsection{Linear separability results}

We establish the main result of this paper, which covers approximately compatible transforms in $\GG_{\lambda,\varepsilon,R}$, the $\varepsilon$ tube around the bounded shifts and scalings $\E_{\lambda,R}$. \Cref{thm:almost_separability} establishes the case for the tube around $\E_{\lambda,R}$, and \Cref{cor:linear_sep_eps} establishes the condition for almost compatible transformations.   In both cases to show linear separability in the LOT embedding space, one must now assume that the two families of distributions are not just disjoint, but actually have a nontrivial minimal distance.    
%
%
%%%%%%%%%%%%%%%%%%%%%%%%%%%%%%%%%%%%%
\begin{theorem}
\label{thm:almost_separability}
%%%%%%%%%%%%%%%%%%%%%%%%%%%%%%%%%%%%%
Let $\mfix,\mfixAdd,\mvariable \in \BB(\RR^n), \mfix,\mfixAdd,\mvariable \ll \lambda$.
Let $R>0,\varepsilon>0$. Consider $\GG \subset \GG_{\lambda,R,\varepsilon}$ and let $\GG$ be convex. Let $\GG \star \mfixAdd$ and $\GG \star \mvariable$ be compact.
If either
\begin{enumerate}
    \item $\mfix$ is the Lebesgue measure on a convex, compact subset of $\RR^n$ or
    \item $\mfix,\mfixAdd,\mvariable$ satisfy the assumptions of Caffarelli's regularity theorem (\Cref{thm:caffarelli-regularity}),
\end{enumerate}
then there exists a $\delta>0$ such that whenever
$W_2(g_1\star \mfixAdd,g_2\star \mvariable)> \delta$ for all $g_1,g_2\in \GG$, we have that $F_{\mfix}(\GG \star\mfixAdd)$ and $F_{\mfix}(\GG \star\mvariable)$ are linearly separable.  Moreover, $\delta$ is computable in both cases (see Remark \ref{remark:delta}) and $\delta\rightarrow 0$ as $\varepsilon \rightarrow 0$.
%%%%%%%%%%%%%%%%%%%%%%%%%%%%%%%%%%%%%
\end{theorem}
%%%%%%%%%%%%%%%%%%%%%%%%%%%%%%%%%%%%%%

%%%%%%%%%%
\begin{proof}
%%%%%%%%%%
We show a detailed proof in \Cref{sec:ProofsLinearSep}.
%%%%%%%%%%
\end{proof}
%%%%%%%%%%%%%%%%%%%%%%%%%%%%%%%%%%%%%%%%%%
\begin{remark}\label{remark:delta}
We note here that for both cases of \Cref{thm:almost_separability}, the sufficient minimal distance $\delta$ can be made explicit:
\begin{enumerate}
    \item In this case, the H{\"o}lder bound by \cite{merigot20} can be used, see \eqref{eq:215bound}. With $\psi(\mfixAdd)=C\|f_{\mfixAdd}\|_{\infty}^{1/15}\varepsilon^{2/15}+\|f_{\mfixAdd}\|_{\infty}^{1/2}\varepsilon$, the choice
    $\delta = 6\,\max\{\psi(\mfixAdd),\psi(\mvariable)\}$ is sufficient. The constant $C$ is the constant as appearing in the derivations by \cite{merigot20}.
    \item In this case, a H{\"o}lder bound following from \cite{gigli-2011} can be used, see \Cref{cor:push_CDT_bound}. With 
    \begin{align*}
        \overline{\psi}(\mfixAdd) := 
             & \left(\sqrt{\frac{4R}{{K_{\mfixAdd}^{\mfix}}}}+2\right)
            \|f_{\mfixAdd}\|_{\infty}^{1/2}\,
        \varepsilon \\
        &+
        \left(4R\,\|f_{\mfixAdd}\|_{\infty}^{1/2}\,\frac{W_2(\mfix,\mfixAdd)+R+\|\id\|_{\mfixAdd}}{K_{\mfixAdd}^{\mfix}}\right)^{1/2}
        \varepsilon^{1/2},
\end{align*}
    the choice $\delta = 6\,\max\{\overline{\psi}(\mfixAdd),\overline{\psi}(\mvariable)\}$ is sufficient. The constant $K_{\mfix}^{\mfixAdd}$ is defined in \Cref{def:strongly_convex}.
\end{enumerate}
Note that a minimal distance $\delta>0$ is needed since we consider perturbations of ``ideal'' functions (shifts and scalings). A version of $\psi$ (respectively $\overline{\psi}$) also appears in characterizing the amount that LOT distance deviates from Wasserstein-2 distance (\Cref{thm:CDT_perturb}).  A parallel of $\psi$ (respectively $\overline{\psi}$) could be established for any approximately compatible transformations by proving a result similar to \Cref{lem:equivalent_strong_convex} for some compatible transformation other than shifts and scalings.

The $\varepsilon$ appears in both $\psi$ and $\overline{\psi}$ since functions in $\GG$ are $\varepsilon$-close to compatible functions, while the $\varepsilon^{2/15}$ respectively $\varepsilon^{1/2}$ come from the general H\"older bounds for LOT as proved in \cite{merigot20} respectively \cite{gigli-2011}.
\end{remark}
%%%%%%%%%%%%%%%%%%%%%%%%%%%%%%%%%%%%%%%%%%%%%

As a corollary to \Cref{thm:almost_separability} with $\varepsilon=0$ and $\delta=0$,
we establish simple conditions under which LOT creates linearly separable sets for distributions in $\BB(\RR^n)$.  This effectively creates a parallel of Theorem \ref{cor:linear_sep_n1} and Theorem 5.6 of \cite{park18} for the higher dimensional cases of LOT, and under the particular compatibility conditions required for higher dimensions.  Theorem \ref{cor:linear_sep_trans_scale} states this for $\E$ (shifts and scalings), and Theorem \ref{cor:linear_sep} in the Appendix provides an equivalent form for subsets of arbitrary compatible transforms. 

%%%%%%%%%%%%%%%%%%%%%%%%%%%%%%%%%%%%%%%%%%%%%%%%%%%%%%%%%
\begin{corollary}\label{cor:linear_sep_trans_scale}
%%%%%%%%%%%%%%%%%%%%%%%%%%%%%%%%%%%%%%%%%%%%%%%%%%%%%%%
Let $\mfix,\mfixAdd,\mvariable \in \BB(\RR^n)$, $\mfix \ll \lambda$, and
let $\HH\subseteq \E$ and let $\HH$ be convex.
If
$\HH \star \mfixAdd$
is
closed
and 
$\HH \star \mvariable$
is compact, and these two sets
are disjoint, then $F_{\mfix}(\HH \star \mfixAdd)$ and $F_{\mfix}(\HH \star \mvariable)$ are linearly separable.
%
%%%%%%%%%%%%%%%%%%%%%%%%%%%%%%%%%%%%%%%%%%%%%
\end{corollary}
%%%%%%%%%%%%%%%%%%%%%%%%%%%%%%%%%%%%%%%%%%%%
%
%
We also note the separability result on $\RR$, which follows directly from the results established above. It is also proved in \cite{park18}.
%
%%%%%%%%%%%%%%%%%%%%%%%%%%%%%%%%%%%%%%%%%%%%%%%%%%%%%%%%%
\begin{corollary}\label{cor:linear_sep_n1}
%%%%%%%%%%%%%%%%%%%%%%%%%%%%%%%%%%%%%%%%%%%%%%%%%%%%%%%
%{\color{red} remark in the end}
Let $\mfix,\mfixAdd,\mvariable \in \BB(\RR)$, $\mfix \ll \lambda$, and
let $\HH$ be a convex set of monotonically increasing functions $\RR \to \RR$.
If
$\HH \star \mfixAdd$
is
closed
and 
$\HH \star \mvariable$
is compact, and these two sets
are disjoint, then $F_{\mfix}(\HH \star \mfixAdd)$ and $F_{\mfix}(\HH \star \mvariable)$ are linearly separable.
%
%%%%%%%%%%%%%%%%%%%%%%%%%%%%%%%%%%%%%%%%%%%
\end{corollary}
%%%%%%%%%%%%%%%%%%%%%%%%%%%%%%%%%%%%%%%%%%%%
%%%%%%%%%%%%%%%%%%%%%%%%%%%%%%%%%%%%%%%%%%%%%%%%%%%%%%%%%%%%%
\begin{remark}\label{rem:generalize_cdt}
%%%%%%%%%%%%%%%%%%%%%%%%%%%%%%%%%%%%%%%%%%%%%%%%%%%%%%%%%%%%
%{\color{red} remark in the end}
Note that \Cref{cor:linear_sep_n1} is also proved in \cite{park18}. In \cite{park18}, $\mathbb{H}$ (equivalent to our $\HH$) is defined as a convex subgroup of the monotonic functions (Definition 5.5 and Definition 5.6 (i)--(iii) of \cite{park18}). We are able to relax the assumption from subgroup to subset, however.
Definition 5.5 of \cite{park18} also assumes differentiability of functions in $\mathbb{H}$, which is needed because constructions are considered from the viewpoint of densities, which means that \eqref{eq:bijective_transport} should hold. Since our approach uses the more general framework of measures rather than densities, we can drop this assumption.
%
\iffalse
The paper \cite{park18} then continues to define $\mathbb{P}$ and $\mathbb{Q}$ (Definition 5.5 (i)), which, are density-based versions of $\HH \star \mfixAdd$ and $\HH \star \mvariable$, respectively.
%
\footnote{
Definition 5.5.\ (i) of \cite{park18} really states that $h\star \mfixAdd \in \mathbb{P}$ for $h \in \HH, \mfixAdd \in \mathbb{P}$ (closedness of $\mathbb{P}$ with respect to the action of $\HH$), and similarly for $\mathbb{Q}$. However, from the proof of Theorem 5.6 it is apparent that the authors mean $\mathbb{P} = \HH \star \mfixAdd$ and $\mathbb{Q} = \HH \star \mvariable$ (with densities instead of measures).
}.
%
Then in Theorem 5.6 of \cite{park18} it is proved that $\cdt_{\mfix}(\HH\star \mfixAdd)$ and $\cdt_{\mfix}(\HH\star \mvariable)$ are linearly separable if $\HH\star \mfixAdd$ and $\HH\star \mvariable$ are disjoint
%
\footnote{
In Theorem 5.6 of \cite{park18} the authors forgot to state the assumption of closedness of $\mathbb{P}$ and compactness of $\mathbb{Q}$ necessary for the Hahn-Banach theorem.}.
%

Theorem 5.6 of \cite{park18} is thus a special case of our \Cref{cor:linear_sep_n1}.
\fi
%%%%%%%%%%%%%%%%%%%%%%%%%%%%%%%%%%%%%%%%%%%%%%%%%%%%%%%%%%%%%%%%
\end{remark}
%%%%%%%%%%%%%%%%%%%%%%%%%%%%%%%%%%%%%%%%%%%%%%%%%%%%%%%%%%%%%%%%%
%

%
%
%%%%%%%%%%%%%%%%%%%%%%%%%%%%%%%%%%
\begin{remark}
We note that each Theorem in this section can be trivially extended to an arbitrary set $\HH$ (or $\GG$) that is not required to be convex, so long as their convex hulls $\conv(\HH)$ (or $\conv(\GG)$) satisfy the needed assumptions of closedness, compactness, and disjointness.
\end{remark}
%%%%%%%%%%%%%%%%%%%%%%%%%%%%%%%%%%%%%%

%%%%%%%%%%%%%%%%%%%%%%%%%%%%%%%%%%%%%%%%%%%%%%%%%%%%%%%%%%%%%%%%%%
\section{Example: Linear separability of MNIST data set}\label{sec:examples}
%%%%%%%%%%%%%%%%%%%%%%%%%%%%%%%%%%%%%%%%%%%%%%%%%%%%%%%%%%%%%%%%%%%%
We linearly separate two classes of digits from the MNIST data set \cite{MNIST} with LOT to verify the linear separability result (\Cref{thm:almost_separability}) numerically. 

We consider the classes of $1$s and $2$s from the MNIST data set. Since the MNIST digits are centered in the middle of the image, and the images have a similar size, we applied an additional (random) scale and shift to every image. Scalings were applied between $0.4$ and $1.2$ using MATLAB's ``imresize'' function. 
These values have been chosen based on the heuristics that smaller scales make some digits unrecognizable and with larger scales some digits are larger than the image.

Within each class, the digits can be considered as shifts, scalings and perturbations of each other. Therefore, via \Cref{thm:almost_separability}, the LOT embedding can be used to separate $1$s from $2$s.

The data consisting of images of $1$s and $2$s is embedded in $L^2$ via the LOT embedding, where we choose as reference density $\mfix$ an isotropic Gaussian. This means that every image $\mfixAdd$ (interpreted as a density on a grid $R \subset \RR^2$) is assigned to the function $T_{\mfix}^{\mfixAdd}: \supp(\mfix) \to R$. Since $\supp(\mfix) \subset R$ is discrete, $T_{\mfix}^{\mfixAdd}(\supp(\mfix))$ is a vector in $\RR^{2n}$, where $n$ is the number of grid points in $\supp(\mfix)$. For each $\mfixAdd$ of the data set, we use this vector as input for the linear classification scheme (we use MATLAB's ``fitcdiscr'' function).

The experiment is conducted in the following way:
We fix the number of testing data to $100$ images from each class (i.e.\ in total, the testing data set consists of $200$ images). Note that we only fix the number of testing data; the actual testing images are chosen randomly from the MNIST data set for each experiment. For the training data set we randomly choose $N$ images from each class, where $N=40,60,80,100$. For each choice $N$, we run $20$ experiments. In each experiment, the classification error of the test data is computed. Then the mean and standard deviation for every $N$ is computed. The mean classification error is shown in \Cref{fig:MNISTseparbilityError} (blue graph labeled ``LOT'') as a function of $N$.

We compare the classification performance of LOT with regular $L^2$ distance between the images. Since we only use a small number of training data ($N=40,60,80,100$ for each digit), and the size of an image is $28 \times 28 = 784$, the dimension of the feature space is much larger than the data point dimension. Such a set-up leads to zero within-class variance in LDA. To prevent this, and in order to allow for a fair comparison, we first apply PCA to reduce the dimension of the images to the same dimension as is used in LOT. The feature space dimension used in LOT is the size of the support of $\mfix$, which consists of $\approx 70$ grid points in these experiments. Thus the dimension is $140$. LDA is then applied to the PCA embeddings of the images. The resulting mean classification error is shown in \Cref{fig:MNISTseparbilityError} (red graph labeled ``PCA'') as a function of $N$.

% %%%%%%%%%%%%%%%%%%%%%%%%%%%%%%%%%%%%%%%%%%%%%%%%%%%%%%%%%%%%%%%%%%%%%%%%%%%%%%%%%%%
% % %FIGURE MNIST Separability between 1s and 2s
% % %%%%%%%%%%%%%%%%%%%%%%%%%%%%%%%%%%%%%%%%%%%%%%%%%%%%%%%%%%%%%%%%%%%%%%%%%%%%%%%%%%%

\begin{figure}[t]
\centering
\begin{picture}(500,200)
% %LABELs
%\put(-0.1,1.4){\large{Classification of MNIST digits with LOT}}
\put(90,0){\large{Number training data for each digit}}
\put(0,40){\rotatebox{90}{\large{Classification error test data}}}
\put(20,50){\rotatebox{90}{\large{(100 images per digit)}}}
%IMAGES
\put(40,20){\includegraphics[scale=0.45]{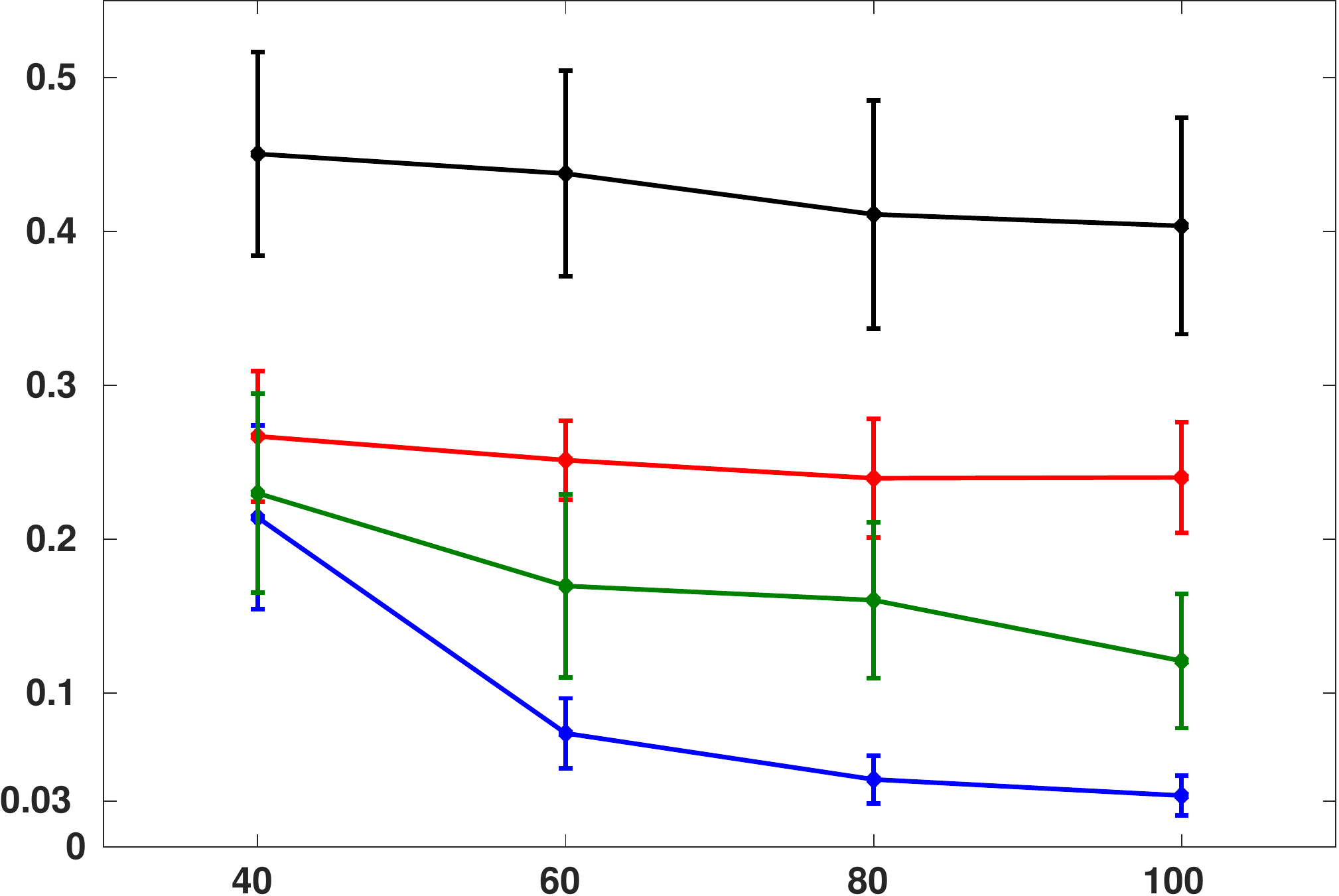}}
\put(310,185){\color{black} Small CNN}
\put(310,170){\color{red} PCA}
\put(310,155){\color{darkgreen} Large CNN}
\put(310,140){\color{blue} LOT}
\end{picture}
\caption{
Classification of MNIST digits $1$s and $2$s with small CNN (black), PCA (red), large CNN (green) and LOT embedding (blue).
We fix the number of testing data (100 images per digit), and vary the number of training data, $N=40,60,80,100$ (the actual training and testing sets are chosen randomly for each experiment). We train a linear classifier on the LOT embedding (blue), the PCA embedding (red) and the two CNNs (black and green) of the training data. The figure shows the mean and standard deviation of the classification error of the testing data over $20$ experiments for each $N$.
}
\label{fig:MNISTseparbilityError}
\end{figure}
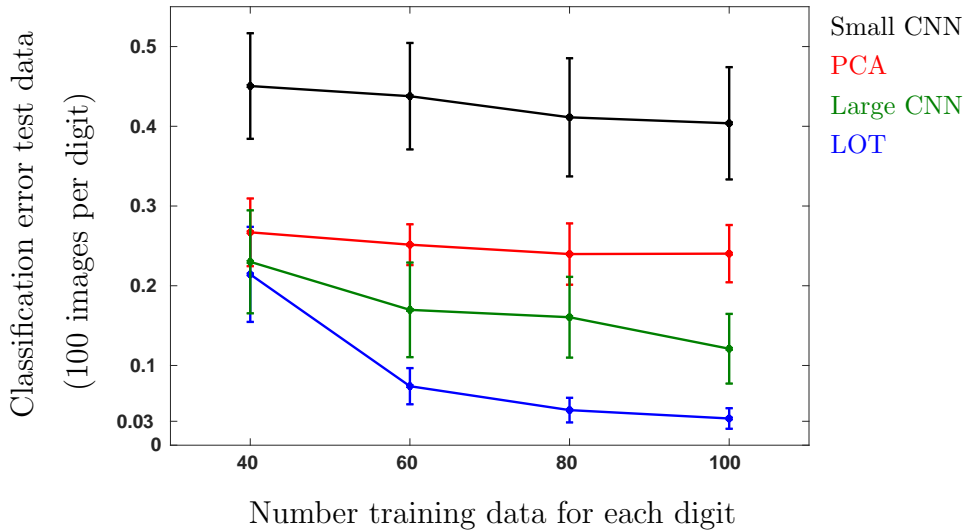
% %%%%%%%%%%%%%%%%%%%%%%%%%%%%%%%%%%%%%%%%%%%%%%%%%%%%%%%%%%%%%%%%%%%%%%%%%%%%%%%%%%%%%%%%

To demonstrate the strengths of embedding into a linearly separable space, we also compare these classification results to training a convolutional neural network (CNN) \cite{goodfellow2016deep} on small amounts of data.  This is not necessarily a perfect comparison, as LOT and PCA are building unsupervised embeddings followed by a supervised classifier in that space, whereas a CNN is building an end-to-end supervised feature extraction and classification.  In theory, this should benefit the CNN if the only method of validation is the overall classification error.  However, as we will demonstrate, in the small data regime the CNN's performance still does not compete with the LOT embedding and linear classification.

To show this, we construct two CNNs to be shown in \Cref{fig:MNISTseparbilityError}.  The first (labeled ``Small CNN'') is a network constructed with three convolutional layers, each with 2 3x3 filters, followed by two fully connected layers, all with ReLU activation units.  In total, this CNN has $182$ trainable parameters, which is of a similar size to the $140$ parameters used in the LOT embedding.  The second (labeled ``Large CNN'') is a similar architecture, but with 8 3x3 filters, and $3650$ trainable parameters.  The CNNs are given the same training data sets as the LOT embeddings, and the testing error is also averaged over $20$ experiments.

It is clear from the figure that the mean error decreases as the number of training data increases for the LOT embedding, while the mean error stagnates for the PCA embedding. Note that we start with a very small amount of training data (40 images from each class), and test on 100 images from each class. The resulting LOT mean error is only $\approx 0.2$. When we train on the same amount as we test (100 images per class), the LOT mean error is already down to $\approx 0.03$.  Similarly, the LOT mean error significantly outperforms both the small and large CNNs.  This is perhaps unsurprising as neural networks are known to require large corpuses of training data \cite{marcus2018deep}, but still serves to demonstrate the strength of embedding into a linearly separable space.

The LOT classification result is also visualized via LDA embedding plots in \Cref{fig:MNIST_LDA} for two experiments. These plots again underline the fact that separation improves as the training data is increased. While training on 100 images per class (right plot of \Cref{fig:MNIST_LDA}) leads to almost perfect separation, training on 40 images per class (left plot of \Cref{fig:MNIST_LDA}) still performs very well considering the small size of the training set.

%%%%%%%%%%%%%%%%%%%%%%%%%%%%%%%%%%%%%%%%%%%%%%%%%%%%%%%%%%%%%%%%%%%%%%%%%%%%%%%%%%%
% %FIGURE MNIST Embedding
% %%%%%%%%%%%%%%%%%%%%%%%%%%%%%%%%%%%%%%%%%%%%%%%%%%%%%%%%%%%%%%%%%%%%%%%%%%%%%%%%%%%

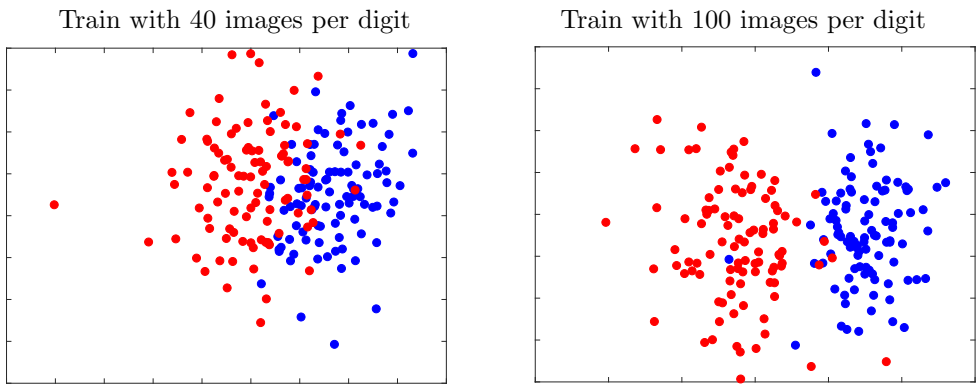
\begin{figure}[t]
\centering
\begin{picture}(500,200)
%LABELs
\put(90,180){\large{LDA embedding of test data}}
\put(10,155){Train with 40 images per digit}
\put(200,155){Train with 100 images per digit}

\put(-10,20){\includegraphics[scale=0.5]{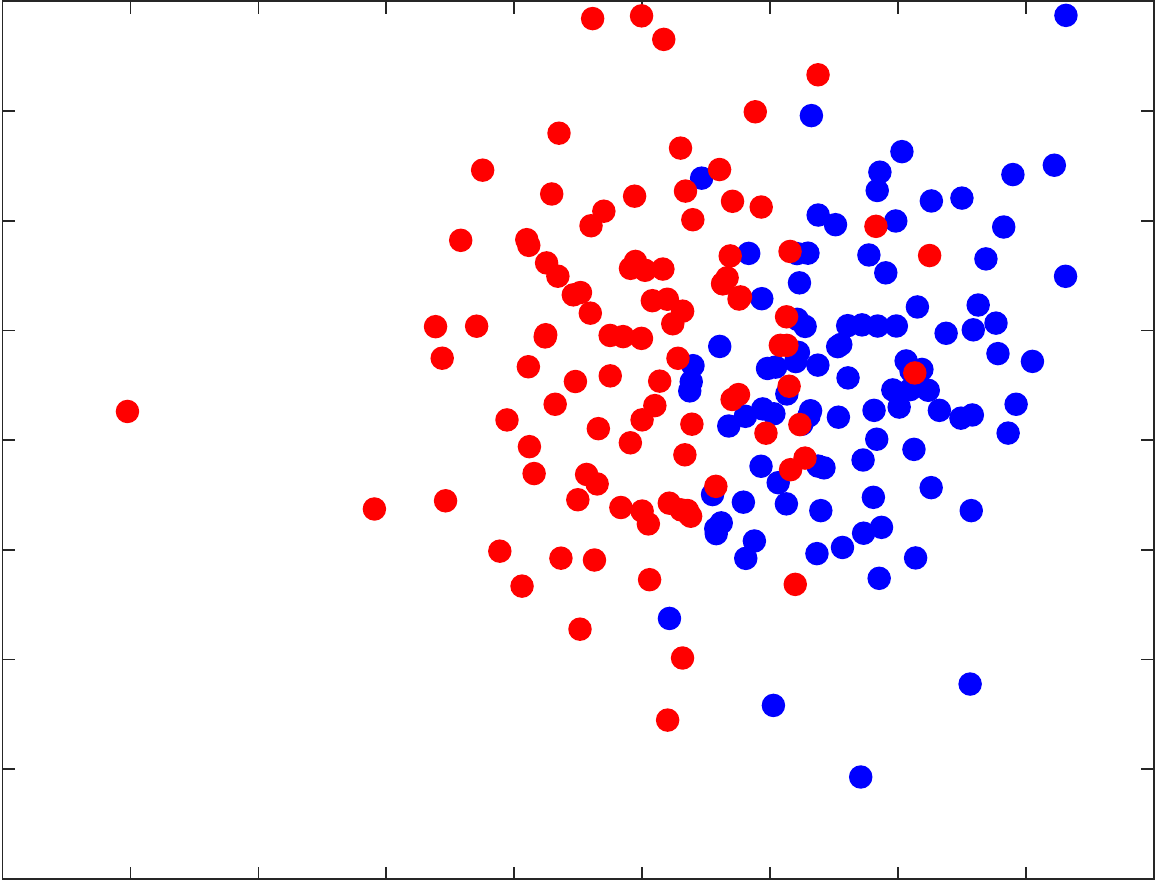}}
\put(190,20){\includegraphics[scale=0.5]{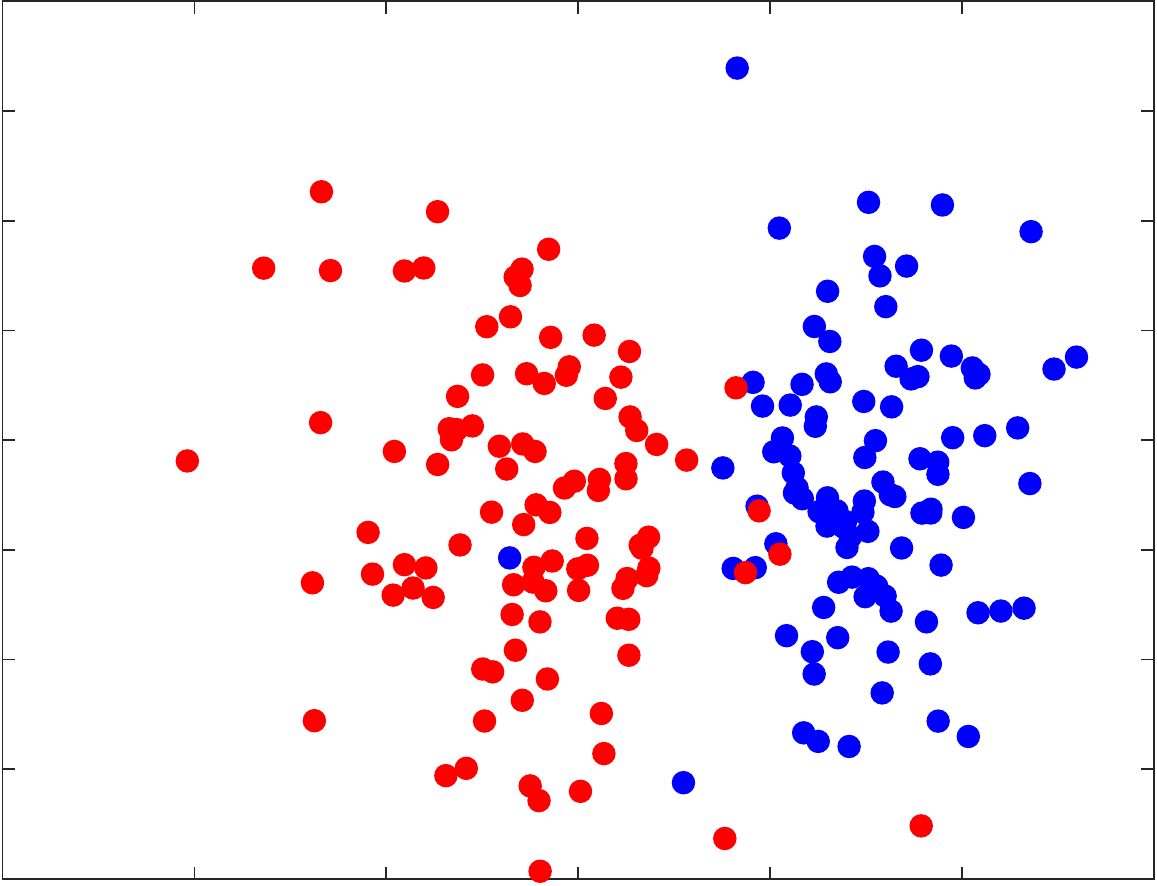}}

\end{picture}
\caption{
LDA embedding plots for the MNIST classification of digits $1$ and $2$ using LOT. As \Cref{fig:MNISTseparbilityError}, these plots underline that the classification improves with the amount of training data.
\emph{Left:} 
We choose one of the experiments carried out for $N=40$ training data for each digit. The testing data (100 images per digit) is embedded in $\RR^2$ through the LDA coordinates. The mean error (\Cref{fig:MNISTseparbilityError}) is $\approx 0.2$, which corresponds to $\approx 40$ digits being misclassified.
\emph{Right:} 
We choose one of the experiments carried out for $N=100$ training data for each digit. The testing data (100 images per digit) is embedded in $\RR^2$ through the LDA coordinates. The mean error (\Cref{fig:MNISTseparbilityError}) is $\approx 0.03$, which corresponds to $\approx 6$ digits being misclassified.
}
\label{fig:MNIST_LDA}
\end{figure}
% %%%%%%%%%%%%%%%%%%%%%%%%%%%%%%%%%%%%%%%%%%%%%%%%%%%%%%%%%%%%%%%%%%%%%%%%%%%%%%%%%%%%%%%%

In addition to the fact that the LOT embedding is capable of producing good separation results on small training data, there is yet another benefit connected to the dimensionality of the problem. To run LDA (or any linear classifier), a matrix of data points versus features needs to be constructed. If we were to compare the original images, the feature space would have dimension equal to the number of grid points. In the LOT embedding only the grid points in $\supp(\mfix)$ need to be considered, rather than the whole grid, which drastically reduces the dimension of the feature space. In the experiments we ran with MNIST, the grid is of size $28 \times 28$, which leads to dimension $28^2 = 784$, while the support of $\mfix$ is $\approx 70$ grid points, hence the dimension is $140$.

This dimension reduction allows us to run LDA on small training data as we did in these experiments. If the feature dimension is very high, one also needs a lot of training data to prevent zero within-class variance, or one has to first apply PCA as we did for \Cref{fig:MNISTseparbilityError}.

%%%%%%%%%%%%%%%%%%%%%%%%%%%%%%%%%%%%%%%%%%%%%%%%%%%%%%%%%%%%%%%%%%%%%
\section*{Acknowledgment}
%%%%%%%%%%%%%%%%%%%%%%%%%%%%%%%%%%%%%%%%%%%%%%%%%%%%%%%%%%%%%%%%%%%%%
AC is supported by NSF DMS grants 1819222 and 2012266, and by Russell Sage Foundation Grant 2196. CM is supported by an AMS-Simons Travel Grant. 
%%%%%%%%%%%%%%%%%%%%%%%%%%%%%%%%%%%%%%%%%%%%%%%%%%%%%%%%%%%%%

%%%%%%%%%%%%%%%%%%%%%%%%%%
%\appendix
\section{Appendix}
%%%%%%%%%%%%%%%%%%%%%%%%%%
%%%%%%%%%%%%%%%%%%%%%%%%%%%%%%%%%%%%%%%%%%%%%%%%%%%%%%%%%%
\subsection{Regularity of the LOT embedding}
\label{sec:regularity}
%%%%%%%%%%%%%%%%%%%%%%%%%%%%%%%%%%%%%%%%%%%%%%%%%%%%%%%%%%

%%%%%%%%%%%%%%%%%%%%%%%%%%%%%%%%%%%%%%%%%%%%%%%%%%%%%%%%%%
The main results of this paper are based on H\"older regularity-type properties of the LOT embedding, which we discuss in more detail in this section.

One of the main ingredients is a version of a theorem on the regularity of the optimal transport map proved by L.\ A.\ Caffarelli \cite{caffarelli92,caffarelli92b,caffarelli96}. The formulation of the theorem is taken from \cite{gigli-2011}:
%%%%%%%%%%%%%%%%%%%%%%%
\begin{theorem}[Caffarelli's regularity theorem]
\label{thm:caffarelli-regularity}
%%%%%%%%%%%%%%%%%%%%%%
Let $\mfix,\mfixAdd \in \BB(\RR^n)$ with $\mfix,\mfixAdd \ll \lambda$. Assume that $\supp(\mfix),\supp(\mfixAdd)$ are $C^2$ and uniformly convex. Further assume that for some $\alpha \in (0,1)$, the densities $f_{\mfix},f_{\mfixAdd}$ are $C^{0,\alpha}$ continuous on their supports and assume that they are bounded from above and below, i.e.\ there exist constants $c,C,\overline{c},\overline{C}>0$ such that
\begin{align*}
    & 0< c \leq \|f_{\mfix}\|_{\infty} \leq C, \\
    & 0< \overline{c} \leq \|f_{\mfixAdd}\|_{\infty} \leq \overline{C}.
\end{align*}
Then $T^{\mfix}_{\mfixAdd}$ is the gradient of a $C^{2,\alpha}$ function on $\supp(\mfixAdd)$.
%
%%%%%%%%%%%%%%%%%%%%%
\end{theorem}
%%%%%%%%%%%%%%%%%%%%%%
%
%
%%%%%%%%%%%%%%%%%%%%%%%%%%%%%%%%%
\begin{definition}
\label{def:strongly_convex}
%%%%%%%%%%%%%%%%%%%%%%%%%%%%%%%%%%
%
We introduce the concept of $k$-strong convexity:
\begin{enumerate}
    \item Let $f:X \to \RR$ with $X\subseteq \RR^n$ convex. 
    $f$ is called $k$-strongly convex if $g_k(x) = f(x)-\frac{1}{2}k\|x\|^2$ is convex.
%    \item The modulus of uniform convexity of $f$ is the supremum $K$ over all $k$ such that $f$ is $k$-strongly convex.
    \item For two measures $\mfix,\mfixAdd \in\BB(\RR^n)$ with $\supp(\mfix)$ convex, denoted by $K_{\mfix}^{\mfixAdd}$ the supremum over all $k$ such that $\varphi$ with $\nabla \varphi = T_{\mfix}^{\mfixAdd}$, is $k$-strongly convex on $\supp(\mfix)$.
\end{enumerate}
%
%%%%%%%%%%%%%%%%%%%%%%%%%%%%%%%%%%
\end{definition}
%%%%%%%%%%%%%%%%%%%%%%%%%%%%%%%%%%
%
In \cite[Corollary 3.2]{gigli-2011} it is proved that if $\mfix,\mfixAdd$ satisfy the assumptions of Caffarelli's regularity theorem (\Cref{thm:caffarelli-regularity}), then $K_{\mfix}^{\mfixAdd}>0$. We further cite the following result from \cite{gigli-2011}:
%
%%%%%%%%%%%%%%%%%%%%%%%%%%%%
\begin{theorem}[{\cite[Proposition 3.3]{gigli-2011}}]
\label{thm:transport_vs_measurepreserving_map}
%%%%%%%%%%%%%%%%%%%%%%%%%%%%
Let $\mfix,\mfixAdd \in \BB(\RR^n)$ and assume they satisfy the same assumptions as in Caffarelli's regularity theorem (\Cref{thm:caffarelli-regularity}).
Then for every $S$ that pushes $\mfix$ to $\mfixAdd$ we have
\begin{equation*}
 \|S-T_{\mfix}^{\mfixAdd}\|_{\mfix}^2 \leq \frac{1}{K_{\mfixAdd}^{\mfix}}
 \left( \|S-\id\|_{\mfix}^2 - W_2(\mfix,\mfixAdd)^2\right).   
\end{equation*}
%
%%%%%%%%%%%%%%%%%%%%%%%%%%%%
\end{theorem}
%%%%%%%%%%%%%%%%%%%%%%%%%%%%
%
Note that in the formulation of this theorem in \cite{gigli-2011}, $2/K_{\mfixAdd}^{\mfix}$ appears instead of $1/K_{\mfixAdd}^{\mfix}$ in the bound. From the proof presented in \cite{gigli-2011} it can be seen, however, that $2$ can be replaced by $1$.

We now prove a bound on the LOT embedding. The proof is based on \Cref{thm:transport_vs_measurepreserving_map} and \cite[Corollary 3.4]{gigli-2011}.
%
%%%%%%%%%%%%%%%%%%%%%%%%%%%%%%%%%%%%%%%%%%%%%%
\begin{theorem}
\label{thm:CDT_bound}
%%%%%%%%%%%%%%%%%%%%%%%%%%%%%%%%%%%%%%%%%%%%%%
Let $\mfix,\mvariable_1,\mvariable_2 \in \BB(\RR^n)$, $\mfix,\mvariable_1,\mvariable_2 \ll \lambda$. Suppose that $\mfix$ and $\mvariable_2$ satisfy the assumptions of Caffarelli's regularity theorem (\Cref{thm:caffarelli-regularity}). Then 
\begin{equation*}
\|F_{\mfix}(\mvariable_1)-F_{\mfix}(\mvariable_2)\|_{\mfix} \leq \left(\frac{2}{{K_{\mvariable_2}^{\mfix}}^{1/2}}+1\right)
W_2(\mvariable_1,\mvariable_2)+
2\left(\frac{W_2(\mfix,\mvariable_2)}{K_{\mvariable_2}^{\mfix}}\right)^{1/2}
W_2(\mvariable_1,\mvariable_2)^{1/2}.
\end{equation*}
%
%%%%%%%%%%%%%%%%%%%%%%%%%%%%%%%%%%%%%%%%%%%%%%%
\end{theorem}
%%%%%%%%%%%%%%%%%%%%%%%%%%%%%%%%%%%%%%%%%%%%%%%%
%%%
\begin{proof}
%%%
Let $S=T_{\mvariable_1}^{\mvariable_2}$. We aim at finding a bound on $\|T_{\mfix}^{\mvariable_1}-T_{\mfix}^{\mvariable_2}\|_{\mfix}$.
%%%

The triangle inequality and change-of-variables formula imply
\begin{align*}
   \|S\circ T_{\mfix}^{\mvariable_1}-T_{\mfix}^{\mvariable_2}\|_{\mfix} & \geq 
   \|T_{\mfix}^{\mvariable_1}-T_{\mfix}^{\mvariable_2}\|_{\mfix} - \|S\circ T_{\mfix}^{\mvariable_1} - T_{\mfix}^{\mvariable_1}\|_{\mfix} \\
   &=
   \|T_{\mfix}^{\mvariable_1}-T_{\mfix}^{\mvariable_2}\|_{\mfix} - \|S- \id\|_{\mvariable_1}\\ & =\|T_{\mfix}^{\mvariable_1}-T_{\mfix}^{\mvariable_2}\|_{\mfix} - W_2(\mvariable_1,\mvariable_2).
\end{align*}
Thus we get
\begin{equation}
\label{eq:CDT_push_Hoelder}
\|T_{\mfix}^{\mvariable_1}-T_{\mfix}^{\mvariable_2}\|_{\mfix}
\leq 
\|S\circ T_{\mfix}^{\mvariable_1}-T_{\mfix}^{\mvariable_2}\|_{\mfix} +
W_2(\mvariable_1,\mvariable_2)
\end{equation}
\Cref{thm:transport_vs_measurepreserving_map} implies 
\begin{equation}
\label{eq:different_transport_approx}
  \|S\circ T_{\mfix}^{\mvariable_1} - T_{\mfix}^{\mvariable_2}\|_{\mfix}^2 \leq \frac{1}{K_{\mvariable_2}^{\mfix}}
  \left(
  \|S\circ T_{\mfix}^{\mvariable_1} - \id\|_{\mfix}^2
  -W_2(\mfix,\mvariable_2)^2
  \right).
\end{equation}
%%
%where $C$ is a constant depending on $T_{\mfix}^{\mvariable_2}$.
%%
Again by the triangle inequality and the change-of-variables formula we have
\begin{align*}
 \|S\circ T_{\mfix}^{\mvariable_1}-\id\|_{\mfix} & \leq 
 \|S\circ T_{\mfix}^{\mvariable_1}-T_{\mfix}^{\mvariable_1}\|_{\mfix}
 + \|T_{\mfix}^{\mvariable_1} - \id\|_{\mfix}
 = W_2(\mvariable_1,\mvariable_2) + W_2(\mfix,\mvariable_1) \\
 & \leq 2\,W_2(\mvariable_1,\mvariable_2) + W_2(\mfix,\mvariable_2)
\end{align*}
Combining this with \eqref{eq:different_transport_approx} we obtain
\begin{align*}
 \|S\circ T_{\mfix}^{\mvariable_1} - T_{\mfix}^{\mvariable_2}\|_{\mfix}^2
 &\leq \frac{4}{K_{\mvariable_2}^{\mfix}}\left( W_2(\mvariable_1,\mvariable_2)^2 + W_2(\mvariable_1,\mvariable_2)W_2(\mfix,\mvariable_2)\right).%\\
\end{align*}
Taking the square root and using the fact that $(a+b)^{1/2}\leq a^{1/2}+b^{1/2}$ we obtain
\begin{equation*}
\|S\circ T_{\mfix}^{\mvariable_1} - T_{\mfix}^{\mvariable_2}\|_{\mfix}
\leq \frac{2}{{K_{\mvariable_2}^{\mfix}}^{1/2}}\left(W_2(\mvariable_1,\mvariable_2) + (W_2(\mvariable_1,\mvariable_2)W_2(\mfix,\mvariable_2))^{1/2}\right).
\end{equation*}
Now \eqref{eq:CDT_push_Hoelder} implies the result.
\end{proof}
%%%
%
%

Note that the ``constants'' in \Cref{thm:CDT_bound} depend on $\mvariable_2$ (namely $K_{\mvariable_2}^{\mfix}$ and $W_2(\mfix,\mvariable_2)$). This can be avoided by considering $\mvariable_2 \in \E\star \mfixAdd$ for a fixed $\mfixAdd \in \BB(\RR^n)$, where $\E$ denotes the set of shifts and scalings. As a preparation for this result, we need the following lemma:
%
%%%%%%%%%%%%%%%%%%%%%%%%%%%%%%%%%%
\begin{lemma}
\label{lem:equivalent_strong_convex}
%%%%%%%%%%%%%%%%%%%%%%%%%%%%%%%%%%
Let $f:X \to \RR$ be differentiable with $X\subseteq \RR^n$ convex. Then we have the following:
\begin{enumerate}
    \item $f$ is $k$-strongly convex on $X$ if and only if $f\circ S_{a}$ is $k$-strongly convex on $S_{a}^{-1}(X)$.
    \item $f$ is $k$-strongly convex on $X$ if and only if $R_c^{-1}\circ f\circ R_{c}$ is $(kc)$-strongly convex on $R_{c}^{-1}(X)$.
\end{enumerate}
%%%%%%%%%%%%%%%%%%%%%%%%%%%%%%%%%%
\end{lemma}
%%%%%%%%%%%%%%%%%%%%%%%%%%%%%%%%%%
%
%%%%%
\begin{proof}
%%%%%
We first note that $X$ is convex if and only if $h^{-1}(X)$ is convex for $h = S_a$ or $h = R_c$. Furthermore, $f$ is $k$-strongly convex if and only if
\begin{equation}\label{eq:strongly-convex-equivalent}
    (\nabla f(x) -\nabla f(y))^T(x-y) \geq k\|x-y\|^2, \quad x,y \in X.
\end{equation}
For $\overline{x},\overline{y} \in S_a^{-1}(X)$, \cref{eq:strongly-convex-equivalent} implies that $f\circ S_a$ is $k$-strongly convex if and only if
\begin{equation*}
    (\nabla f\circ S_a (\overline{x})- \nabla f\circ S_a (\overline{y}))^T(\overline{x}-\overline{y})
    \geq  k\|\overline{x}-\overline{y}\|^2
\end{equation*}
which is the same as
\begin{equation*}
    (\nabla f (S_a (\overline{x}))- \nabla f (S_a (\overline{y})))^T(S_a(\overline{x})-S_a(\overline{y}))
    \geq k\|S_a(\overline{x})-S_a(\overline{y})\|^2.
\end{equation*}
As this is only a transformation $x=S_a(\overline{x})$ and $y=S_a(\overline{y})$ compared to \cref{eq:strongly-convex-equivalent}, $k$-strong convexity of $f$ and $f\circ S_a$ are equivalent.

For $\overline{x},\overline{y} \in R_c^{-1}(X)$,  \cref{eq:strongly-convex-equivalent} implies that $R_c^{-1} \circ f\circ R_c$ is $(kc)$-strongly convex if and only if
\begin{equation*}
    (\nabla (R_c^{-1}\circ f\circ R_c) (\overline{x})- \nabla (R_c^{-1}\circ f\circ R_c) (\overline{y}))^T(\overline{x}-\overline{y})
    \geq  kc \,\|\overline{x}-\overline{y}\|^2
\end{equation*}
which is the same as
\begin{equation*}
   c^{-1} (\nabla f (R_c (\overline{x}))- \nabla f (R_c (\overline{y})))^T(R_c(\overline{x})-R_c(\overline{y}))
    \geq kc\,c^{-2}\|R_c(\overline{x})-R_c(\overline{y})\|^2
\end{equation*}
resulting in
\begin{equation*}
    (\nabla f (R_c (\overline{x}))- \nabla f (R_c (\overline{y})))^T(R_c(\overline{x})-R_c(\overline{y}))
    \geq k\|R_c(\overline{x})-R_c(\overline{y})\|^2
\end{equation*}
As this is only a transformation $x=R_c(\overline{x})$ and $y=R_c(\overline{y})$ compared to \cref{eq:strongly-convex-equivalent}, $k$-strong convexity of $f$ and $(kc)$-strong convexity of $R_c^{-1}\circ f\circ R_c$ are equivalent.
%%%%%%%%%%%%%%
\end{proof}
%%%%%%%%%%%%%
%
%
%%%%%%%%%%%%%%%%%%%%%%%%%%%%%%%%%%%%%%%%%%%%%
\begin{corollary}
\label{cor:push_CDT_bound}
%%%%%%%%%%%%%%%%%%%%%%%%%%%%%%%%%%%%%%%%%%%%%
Let $\mfix,\mfixAdd \in \BB(\RR^n)$, $\mfix,\mfixAdd \ll \lambda$. Further assume that $\mfix$ and $\mfixAdd$ satisfy the assumptions of Caffarelli's regularity theorem (\Cref{thm:caffarelli-regularity}). Let $R>0$ and consider $h\in \E_{\mfixAdd,R}$ (bounded shifts/scalings, see \cref{eq:E}) as well as $g\in \FFm$. Then we have
\begin{align*}
\|F_{\mfix}(g_{\sharp}\mfixAdd)-F_{\mfix}(h_{\sharp}\mfixAdd )\|_{\mfix} \leq & \left(\sqrt{\frac{4R}{{K_{\mfixAdd}^{\mfix}}}} +1\right) 
W_2(g_{\sharp}\mfixAdd,h_{\sharp}\mfixAdd) \\
& +\left(4R\,\frac{W_2(\mfix,\mfixAdd)+R+\|\id\|_{\mfixAdd}}{K_{\mfixAdd}^{\mfix}}\right)^{1/2}\,
W_2(g_{\sharp}\mfixAdd,h_{\sharp}\mfixAdd)^{1/2}.
\end{align*}
%%
%%%%%%%%%%%%%%%%%%%%%%%%%%%%%%%%%%%%%%%%%%%%%
\end{corollary}
%%%%%%%%%%%%%%%%%%%%%%%%%%%%%%%%%%%%%%%%%%%%%
%
Note that we now have a bound with constants that do not depend on $h$ or $g$. They only depend on the fixed measures $\mfix,\mfixAdd$ and on the radius $R$. 
%

%%%
\begin{proof}
%%%
Let $\mvariable_1=g_{\sharp}\mfixAdd$ and $\mvariable_2=h_{\sharp}\mfixAdd$. First note that since $\mfixAdd$ and $\mfix$ satisfy the assumptions of Caffarelli's regularity theorem, also $\mvariable_2$ and $\mfix$ satisfy them. Therefore we can apply \Cref{thm:CDT_bound}.
%%%

We now bound $W_2(\mfix,\mvariable_2)$ and $K_{\mvariable_2}^{\mfix}$ from \Cref{thm:CDT_bound} by constants that only depend on $\mfix,\mfixAdd$ and $R$. Such bounds then imply the result.
%%%

The triangle inequality, \Cref{lem:translate_transport} and the assumption $\|h\|_{\mfixAdd}<R$ imply
\begin{align*}
    W_2(\mfix,\mvariable_2) 
    & \leq W_2(\mfix,\mfixAdd) + W_2(\mfixAdd,h_{\sharp}\mfixAdd) 
    = W_2(\mfix,\mfixAdd) + \|T_{\mfixAdd}^{h_{\sharp}\mfixAdd}-\id\|_{\mfixAdd}\\
    &= W_2(\mfix,\mfixAdd) + \|h-\id\|_{\mfixAdd} \\
    & < W_2(\mfix,\mfixAdd) +R+\|\id\|_{\mfixAdd}.
\end{align*}

We now show that $K_{\mvariable_2}^{\mfix}$ only depends on $\mfix,\mfixAdd$ and $R$, but does not depend on $h$.
First consider $h = S_a$.
Note that $T_{\mvariable_2}^{\mfix} = T_{\mfixAdd}^{\mfix}\circ S_a^{-1}$ and $T_{\mfixAdd}^{\mfix}=\nabla \psi$ implies $T_{\mfixAdd}^{\mfix}\circ S_a^{-1} = \nabla \psi\circ S_a^{-1}$. Also, $\psi\circ S_a^{-1}$ is convex on $S_a(\supp(\mfixAdd))$. This implies that $\varphi = \psi\circ S_a^{-1}$.

\Cref{lem:equivalent_strong_convex} implies that $\psi$ is $k$-strongly convex if and only if $\varphi =\psi\circ S_{a}^{-1}$ is $k$-strongly convex. Therefore the modulus of uniform convexity of $\psi \circ S_a^{-1}$ equals the modulus of uniform convexity of $\psi$. Thus $K_{\mvariable_2}^{\mfix}=K_{\mfixAdd}^{\mfix}$, which is independent of $S_{a}$.

Now consider $h=R_c$. Again we have $T_{\mvariable_2}^{\mfix} = T_{\mfixAdd}^{\mfix}\circ R_c^{-1}$ and $T_{\mfixAdd}^{\mfix}=\nabla \psi$ implies $T_{\mfixAdd}^{\mfix}\circ R_c^{-1} = \nabla R_c\circ \psi\circ R_c^{-1}$. Also, $R_c\circ \psi\circ R_c^{-1}$ is convex on $R_c(\supp(\mfixAdd))$. This implies that $\varphi = R_c\circ \psi\circ R_c^{-1}$.

\Cref{lem:equivalent_strong_convex} implies that $\psi$ is $k$-strongly convex if and only if $\varphi=R_c\circ \psi\circ R_{c}^{-1}$ is $kc^{-1}$-strongly convex. Therefore $K_{\mvariable_2}^{\mfix}=K_{\mfixAdd}^{\mfix}c^{-1}$. Since by assumption $|c| = \|R_{c}\|_{\mfixAdd}<R$ we have
\begin{equation*}
 \frac{1}{K_{\mvariable_2}^{\mfix}}
= \frac{1}{K_{\mfixAdd}^{\mfix}}|c|<\frac{R}{K_{\mfixAdd}^{\mfix}},
\end{equation*}
which gives a bound independent of $R_{c}$.
%%%%
\end{proof}
%%%

We now combine \Cref{cor:push_CDT_bound} with the Lipschitz continuity of the pushforward map $g \mapsto g_{\sharp}\mfix$ to obtain a H\"older regularity-type result for LOT. We first cite the result on the Lipschitz continuity of the pushforward map, which can be found in e.g.\ \cite[Equation (2.1)]{ambrosio-2013}:
%%%
\begin{equation}\label{thm:pushforward_Lipschitz}
    W_2(g_{\sharp}\mfixAdd,h_{\sharp}\mfixAdd) \leq \|g-h\|_{\mfixAdd}.
\end{equation}
%%%
%%%%%%%%%%%%%%%%%%%%%%%%%%%%%%%%%%%%%%%%%%%%%%%%%%%%
\begin{corollary}
\label{cor:push_CDT_bound_L2}
%%%%%%%%%%%%%%%%%%%%%%%%%%%%%%%%%%%%%%%%%%%%%%%%%%%%
Let $\mfix,\mfixAdd \in \BB(\RR^n)$, $\mfix,\mfixAdd \ll \lambda$. Further assume that $\mfix$ and $\mfixAdd$ satisfy the assumptions of Caffarelli's regularity theorem (\Cref{thm:caffarelli-regularity}). Let $R>0$, $h\in \E_{\mfixAdd,R}$ (see \cref{eq:E}) and $g\in L^2(\RR^n,\mfixAdd)$. Then we have
\begin{align*}
\|F_{\mfix}(g_{\sharp}\mfixAdd)-F_{\mfix}(h_{\sharp}\mfixAdd )\|_{\mfix}
  \leq &
\left(\sqrt{\frac{4R}{{K_{\mfixAdd}^{\mfix}}}}+1\right)\,\|g-h\|_{\mfixAdd} \\
&+
\sqrt{4R\,\frac{W_2(\mfix,\mfixAdd)+R+\|\id\|_{\mfixAdd}}{K_{\mfixAdd}^{\mfix}}}\,
\|g-h\|^{1/2}_{\mfixAdd}.
\end{align*}
%%
%%%%%%%%%%%%%%%%%%%%%%%%%%%%%%%%%%%%%%%%%%%%%%%%%%%%%
\end{corollary}
%%%%%%%%%%%%%%%%%%%%%%%%%%%%%%%%%%%%%%%%%%%%%%%%%%%%
%%
%
\begin{remark}
In \cite[Corollary 3.4]{gigli-2011} it is proved that for fixed $\mfix$ and a Lipschitz continuous curve $\mfixAdd_t$ of absolutely continuous measures, $t\in [0,1]$, $1/2$-H\"older regularity of $t \mapsto F_{\mfix}(\mfixAdd_t)$ can be achieved. Indeed, it is proved that
\begin{equation*}
    \|F_{\mfix}(\mfixAdd_t) - F_{\mfix}(\mfixAdd_0)\|_{\mfix}\leq C \sqrt{t}.
\end{equation*}
\Cref{cor:push_CDT_bound_L2} can be considered a generalization of this result. We prove that the map $h \mapsto F_{\mfix}(h_{\sharp}\mfixAdd)$ can achieve H\"older-type regularity between an element of $\E$ (comparable to $\mfixAdd_0$) and an element of $L^2(\RR^n,\mfixAdd)$ (comparable to $\mfixAdd_t$). Note that like $\mfixAdd_t$, the ``curve'' $h\mapsto h_{\sharp}\mfixAdd$ is Lipschitz continuous (\cref{thm:pushforward_Lipschitz}). The restriction to bounded shifts and scalings (via $R>0$) relates to the fact that $[0,1]$ is bounded.
\end{remark}
%%
%%%%%%%%%%%%%%%%%%%%%%%%%%%%%%%%%%%%%%%%%%%%%%%%%%%%%%%
\subsection{Set-up for linear separability results}\label{appendix:lin_sep}
%%%%%%%%%%%%%%%%%%%%%%%%%%%%%%%%%%%%%%%%%%%%%%%%%%%%%%%
%
%
In this section we build up the theory needed for the results on linear separability presented in \Cref{sec:linear_sep}. The proofs for these results can then be derived easily from results of this section, see \Cref{sec:ProofsLinearSep}.

Throughout this section, let $\HH \subseteq \FF$.
Then $\HH$ acts on $\BB(\RR^n)$ by push-forward
\begin{equation*}
    h \star \mfixAdd = h_{\sharp}\mfixAdd, \qquad h \in \HH, \mfixAdd \in \BB(\RR^n).
\end{equation*}
This is a group action if $\HH$ is a subgroup of $\FF$.

Fix $\mu \in \BB(\RR^n)$. Using the notation from \Cref{lem:isometry_compatible}, we denote by
\begin{equation*}
\HH \star \mfixAdd =\left\{h\star \mfixAdd: h\in \HH \right\}
\end{equation*}
the orbit of $\mu$ with respect to the action of $\HH$.

Note that $\HH$ also acts on $\FF$ by composition, i.e. $h\star f = h \circ f$ for $f \in \FF$ and $h \in \HH$. We also denote this action by $\star$.

We now derive some properties of this action in connection with the LOT embedding $F_{\sigma}$.
%

%%%%%%%%%%%%%%%%%%%%%%%%%%%%%%%%%%%%%%%%%%%%%%%%%%%%%%%%%%
\begin{definition}
\label{def:orbit_compatibility}
%%%%%%%%%%%%%%%%%%%%%%%%%%%%%%%%%%%%%%%%%%%%%%%%%%%%%%%%%
Let $\mfix,\mfixAdd \in \BB(\RR^n)$, $\mfix \ll \lambda$, and let $\HH \subseteq \FF$. We say that $F_{\mfix}$ is \emph{compatible with $\mu$-orbits with respect to the action of $\HH$} if
\begin{equation}\label{eq:sets_equal2}
    F_{\mfix}(h \star \mfixAdd) = h \star F_{\mfix}(\mfixAdd),
    \qquad  h\in \HH.
\end{equation}
%%%%%%%%%%%%%%%%%%%%%%%%%%%%%%%%%%%%%%%%%%%%%%%%%%%%
\end{definition}
%%%%%%%%%%%%%%%%%%%%%%%%%%%%%%%%%%%%%%%%%%%%%%%%%%%%%
%
%
%%%%%%%%%%%%%%%%%%%%%%%%%%%%%%%%%%%%%%%%%%%%%%%%%%%%%%%
\begin{remark}
\label{rem:dropCompatibility}
%%%%%%%%%%%%%%%%%%%%%%%%%%%%%%%%%%%%%%%%%%%%%%%%%%%%%%%%
Note that \Cref{eq:sets_equal2} is exactly \cref{eq:compatible}. We just introduced a new notation via $\star$.

As is shown in Lemma \ref{lem:translate_transport}, Condition \eqref{eq:sets_equal2} is satisfied by shifts and scalings in arbitrary dimension, and by all monotonically increasing functions in dimension $n=1$.  
%
%%%%%%%%%%%%%%%%%%%%%%%%%%%%%%%%%%%%%%%%%%%%%%%%%%%%%%%%
\end{remark}
%%%%%%%%%%%%%%%%%%%%%%%%%%%%%%%%%%%%%%%%%%%%%%%%%%%%%%%%
%
%
A version of the following lemma is also proved in \cite{aldroubi20}.
%%%%%%%%%%%%%%%%%%%%%%%%%%%%%%%%%%%%%
\begin{lemma}\label{lem:convex-image}
%%%%%%%%%%%%%%%%%%%%%%%%%%%%%%%%%%%%
Let $\mfix, \mfixAdd \in \BB(\RR^n), \mfix \ll \lambda$, and let $\HH \subseteq \FF$ be convex. If $F_{\mfix}$ is compatible with $\mfixAdd$-orbits with respect to the action of $\HH$ (\Cref{def:orbit_compatibility}) then $F_{\mfix}(\HH \star \mfixAdd)$ is convex.
%%%%%%%%%%%%%%%%%%%%%%%%%%%%%%%%%%%%%%%%%%%%%
\end{lemma}
%%%%%%%%%%%%%%%%%%%%%%%%%%%%%%%%%%%%%%%%%%%%%
%%%%%%%%%%%%%%%%%%%%%%%%%%%%%%%%%%%%%%%%%%%%%
\begin{proof}
%%%%%%%%%%%%%%%%%%%%%%%%%%%%%%%%%%%%%%%%%%%%
%
We prove that for $f\in \FF$, convexity of $\HH$ implies convexity of $\HH\star f$. This together with condition \eqref{eq:sets_equal2} then implies convexity of $F_{\mfix}(\HH\star\mfixAdd)$.

Let $c \in [0,1]$ and let $h_1,h_2 \in \HH$. Then
\begin{equation*}
 (1-c) (h_1 \circ f) + c\,( h_2 \circ f)      
  = ((1-c)h_1 + c\,h_2) \circ f \in \HH \star f. \qedhere
\end{equation*}
%
%%%%%%%%%%%%%%%%%%%%%%%%%%%%%%%%%%%%%%%%%
\end{proof}
%%%%%%%%%%%%%%%%%%%%%%%%%%%%%%%%%%%%%%%%%%
%%%%%%%%%%%%%%%%%%%%%%%%%%%%%%%%%%%%%%%%%%%%%%%%%%%%%%%%%
\begin{theorem}\label{cor:linear_sep}
%%%%%%%%%%%%%%%%%%%%%%%%%%%%%%%%%%%%%%%%%%%%%%%%%%%%%%%
Let $\mfix,\mfixAdd,\mvariable \in \BB(\RR^n)$, $\mfix \ll \lambda$ and
let $\HH\subseteq\FF$ be convex.
Further assume that $F_{\mfix}$ is compatible with both $\mfixAdd$- and $\mvariable$-orbits with respect to the action of $\HH$ (\Cref{def:orbit_compatibility}).
If
$\HH \star \mfixAdd$
is
closed
and 
$\HH \star \mvariable$
is compact, and these two sets are disjoint, then $F_{\mfix}(\HH \star \mfixAdd)$ and $F_{\mfix}(\HH \star \mvariable)$ are linearly separable.
\end{theorem}
%%%%%%%%%%%%%%%%%%%%%%%%%%%%%%%%%%%%%%%%%%%%
\begin{proof}
%%%%%%%%%%%%%%%%%%%%%%%%%%%%%%%%%%%%%%%%%%%
Since $F_{\sigma}$ is continuous, 
$F_{\sigma}\left( \HH \star \mfixAdd \right)$
is compact and 
$F_{\sigma}(\HH \star \mvariable)$ 
is closed. 
Since $F_{\sigma}$ is injective (\Cref{lem:injective}), they are also disjoint. \Cref{lem:convex-image} implies that both images are convex. Therefore, the Hahn-Banach Theorem implies separability.
\end{proof}
%%%%%%%%%%%%%%%%%%%%%%%%%%%%%%%%%%%%%%%%%%
%
%
%

\Cref{def:orbit_compatibility} is a strong condition which is satisfied for shifts and scalings. In the following we show a linear separability result which relaxes this condition. Indeed, we show that \Cref{cor:linear_sep} is also true if we extend $\HH$ by functions which are $\varepsilon$-close to shifts and scalings in $\FF$. 
In analogy to \Cref{def:orbit_compatibility} we define compatibility of $F_{\mfix}$ with respect to $\mfixAdd$-orbits up to an error $\varepsilon$.
%%%%%%%%%%%%%%%%%%%%%%%%%%%%%%%%%%%%%%%%%%%%%%%%%%%%%%%%%%
\begin{definition}
\label{def:orbit_compatibility_eps}
%%%%%%%%%%%%%%%%%%%%%%%%%%%%%%%%%%%%%%%%%%%%%%%%%%%%%%%%%
Let $\mfix,\mfixAdd \in \BB(\RR^n)$, $\mfix \ll \lambda$, let $\HH \subseteq \FF$, and let $\varepsilon>0$. We say that $F_{\mfix}$ is \emph{$\varepsilon$-compatible with $\mu$-orbits with respect to the action of $\HH$} if
\begin{equation*}%\label{eq:sets_equal2}
    \|F_{\mfix}(h \star \mfixAdd) - h \star F_{\mfix}(\mfixAdd)\|_{\mfix}<\varepsilon
    \qquad  h\in \HH.
\end{equation*}
%%%%%%%%%%%%%%%%%%%%%%%%%%%%%%%%%%%%%%%%%%%%%%%%%%%%
\end{definition}
%%%%%%%%%%%%%%%%%%%%%%%%%%%%%%%%%%%%%%%%%%%%%%%%%%%%%
There is also an analog to \Cref{lem:convex-image}:
%
%%%%%%%%%%%%%%%%%%%%%%%%%%%%%%%%%%%%%
\begin{lemma}\label{lem:convex-image-eps}
%%%%%%%%%%%%%%%%%%%%%%%%%%%%%%%%%%%%
Let $\mfix, \mfixAdd \in \BB(\RR^n), \mfix \ll \lambda$, let $\HH \subseteq \FF$ be convex, and let $\varepsilon>0$. If $F_{\mfix}$ is $\varepsilon$-compatible with $\mfixAdd$-orbits with respect to the action of $\HH$ (\Cref{def:orbit_compatibility_eps}) then $F_{\mfix}(\HH \star \mfixAdd)$ is $2\varepsilon$-convex (\Cref{def:eps_convex}).
%%%%%%%%%%%%%%%%%%%%%%%%%%%%%%%%%%%%%%%%%%%%%
\end{lemma}
%%%%%%%%%%%%%%%%%%%%%%%%%%%%%%%%%%%%%%%%%%%%%
%
%%%
\begin{proof}
%%%
Let $h_1,h_2\in \HH$ and $c\in [0,1]$. Define $h = (1-c)h_1 + ch_2\in \HH$. We aim at proving that
\begin{equation*}
    \|(1-c)F_{\mfix}(h_1\star \mfixAdd) + cF_{\mfix}(h_2\star \mfixAdd) - F_{\mfix}(h\star \mfixAdd)\|_{\mfix} < 2\varepsilon.
\end{equation*}
To this end, we apply \Cref{def:orbit_compatibility_eps}:
\begin{align*}
\|(1-c)&F_{\mfix}(h_1\star \mfixAdd) + cF_{\mfix}(h_2\star \mfixAdd)  - F_{\mfix}(h\star \mfixAdd)\|_{\mfix} 
\leq \\
(1-c)&\|F_{\mfix}(h_1\star \mfixAdd) -h_1\star F_{\mfix}(\mfixAdd)\|_{\mfix} 
+
c\|F_{\mfix}(h_2\star \mfixAdd) - h_2\star F_{\mfix}(\mfixAdd)\|_{\mfix} \\
& +
\|h\star F_{\mfix}(\mfixAdd)-F_{\mfix}(h\star \mfixAdd)\|_{\mfix} \\
& < (1-c)\varepsilon+c\varepsilon + \varepsilon = 2\varepsilon. \qedhere
\end{align*}
%%
%%%
\end{proof}
%%%
%

This lemma allows us to establish the most general form of the linear separability theorem, which simply requires the additional assumption that the two families generated by action $\HH$, $\HH\star\mfixAdd$ and $\HH \star \mvariable$, have a minimal distance greater than $6\varepsilon$.
%%%%%%%%%%%%%%%%%%%%%%%%%%%%%%%%%%%%%%%%%%%%%%%%%%%%%%%%%
\begin{theorem}\label{cor:linear_sep_eps}
%%%%%%%%%%%%%%%%%%%%%%%%%%%%%%%%%%%%%%%%%%%%%%%%%%%%%%%
Let $\mfix,\mfixAdd,\mvariable \in \BB(\RR^n)$, $\mfix \ll \lambda$,
let $\HH\subseteq\FF$ be convex, and let $\varepsilon>0$.
Further assume that $F_{\mfix}$ is $\varepsilon$-compatible with both $\mfixAdd$- and $\mvariable$-orbits with respect to the action of $\HH$ (\Cref{def:orbit_compatibility_eps}).
If
$\HH \star \mfixAdd$
and 
$\HH \star \mvariable$
are compact, and $W_2(h_1\star\mfixAdd,h_2\star \mvariable)>6\varepsilon$ for all $h_1,h_2\in \HH$, then $F_{\mfix}(\HH \star \mfixAdd)$ and $F_{\mfix}(\HH \star \mvariable)$ are linearly separable.
%
%%%%%%%%%%%%%%%%%%%%%%%%%%%%%%%%%%%%%%%%%%%%
\end{theorem}
%%%%%%%%%%%%%%%%%%%%%%%%%%%%%%%%%%%%%%%%%%%%
%
%
%%%%
\begin{proof}
%%%
Since $F_{\mfix}$ is continuous, both $A = F_{\mfix}(\HH \star \mfixAdd)$ and $B = F_{\mfix}(\HH \star \mvariable)$ are compact. Now consider the closed convex hull of these sets, i.e.\ consider $\overline{\conv(A)}$ and $\overline{\conv(B)}$. The closed convex hull of compact sets is compact again in a completely metrizable locally convex space \cite[Theorem 5.35]{aliprantis06}. Thus, in order to apply the Hahn-Banach theorem to $\overline{\conv(A)}$ and $\overline{\conv(B)}$, we only need to show that these sets are disjoint.

\Cref{lem:W2-CDT-approx} implies 
\begin{equation*}
    6\varepsilon < W_2(h_1\star \mfixAdd,h_2\star \mvariable) \leq 
    \|F_{\mfix}(h_1\star \mfixAdd)-F_{\mfix}(h_2\star \mvariable)\|_{\mfix},
\end{equation*}
for $h_1,h_2\in \HH$. Therefore  $d(A,B)>6 \varepsilon$, where $d$ denotes the distance between sets. 

Since $F_{\mfix}$ is $\varepsilon$-compatible with respect to both $\mfixAdd$- and $\mvariable$-orbits, \Cref{lem:convex-image-eps} implies that both $A$ and $B$ are $2\varepsilon$-convex (\Cref{def:eps_convex}). This means that $d(\conv(A),A)<2\varepsilon$ and $d(\conv(B),B)<2\varepsilon$. 

\Cref{lem:distance_convex_hull} now implies that $d(\conv(A),\conv(B))>\varepsilon$. Therefore the closure of these sets has positive distance, $d(\overline{\conv(A)},\overline{\conv(B)})>0$, which implies that $\overline{\conv(A)}\cap \overline{\conv(B)} = \emptyset$.
%%%
\end{proof}
%%%%
%
%
%%%%%%%%%%%%%%%%%%%%%%%%%%%%%%%%%%%%%%%%%%%%%%%%%%%%%%%
\subsection{Proofs of \Cref{sec:IntroLOT}}
\label{sec:ProofsIntroLOT}
%%%%%%%%%%%%%%%%%%%%%%%%%%%%%%%%%%%%%%%%%%%%%%%%%%%%%%%
%
%%%%%%%%%%%%%%%%%%%%%%%%%%%%%%%%%%%%%%%%%%%%%%%%%%%%%%%%%
\begin{proof}[Proof of \Cref{lem:injective}]
%%%%%%%%%%%%%%%%%%%%%%%%%%%%%%%%%%%%%%%%%%%%%%%%%%%%%%%%%
To prove part 1 of the lemma, we show continuity and injectivity of $F_{\mfix}$.
%; both properties then also follow for $\cdt_{\mfix}$.

The stability of transport maps as described in \cite[Corollary 5.23]{villani-2009} implies that $F_{\mfix}$ is continuous. 

If $F_{\mfix}(\mvariable_1) = F_{\mfix}(\mvariable_2)$, then $T_{\mfix}^{\mvariable_1} = T_{\mfix}^{\mvariable_2}$. In particular this implies
\begin{equation*}
    \mvariable_1 = {T_{\mfix}^{\mvariable_1}}_{\sharp}\mfix = {T_{\mfix}^{\mvariable_2}}_{\sharp}\mfix = \mvariable_2. 
\end{equation*}
This implies injectivity of $F_{\mfix}$.
%%%
%%%

%%%

To prove part 2 of the lemma,
let $c \in [0,1]$ and let $\mvariable_1,\mvariable_2 \in \BB(\RR^n)$. We define
\begin{equation*}
    T(x) := (1-c)\, F_{\mfix}(\mvariable_1)(x) + c \, F_{\mfix}(\mvariable_2)(x), \quad x \in \RR^n.
\end{equation*}
We need to show that there exists $\mvariable_3\in \BB(\RR^n)$ such that
%
%
%\begin{equation*}
$
    T = F_{\mfix}(\mvariable_3).
$    
%\end{equation*}
%
%
To this end, we define 
%\begin{equation*}
$
    \mvariable_3 := T_{\sharp}\mfix.
$    
%\end{equation*}
%
By definition, $T$ pushes $\mfix$ to $\mvariable_3$. We now show that $T$ can be written as the gradient of a convex function.

By \Cref{thm:brenier} there exist convex functions $\varphi_1,\varphi_2$ such that $T_{\mfix}^{\mvariable_1}$ and $T_{\mfix}^{\mvariable_2}$ can be written uniquely as $T_{\mfix}^{\mvariable_j}(x) = \nabla \varphi_j(x)$, $j=1,2, x\in \RR^n$. 
This implies that $T(x) = \nabla \varphi_3(x)$, with the convex function 
\begin{equation*}
    \varphi_3(x) = (1-c)\, \varphi_1(x) +c\, \varphi_2(x), \quad x \in \RR^n.
\end{equation*}
\Cref{thm:brenier} thus implies that $T = T_{\mfix}^{\mvariable_3}$, which proves $T = F_{\mfix}(\mvariable_3)$.
%%%%%%%%%%%%%%%%%%
\end{proof}
%%%%%%%%%%%%%%%%%%%%%

%%%%%%%%%%%%%%%%%%%%%%%%%%%%%%%%%%%%%%%%%%%%%%%%%%%%%%%%%%%%%%%%%%%%%%%%
\begin{proof}[Proof of \Cref{lem:translate_transport}]
%%%%%%%%%%%%%%%%%%%%%%%%%%%%%%%%%%%%%%%%%%%%%%%%%%%%%%%%%%%%%%%%%%%%%%%
On $\RR$ recall that
\begin{equation}\label{eq:OMT-1dim}
 T_{\mfix}^{\mvariable} = G_{\mvariable}^{-1}\circ G_{\mfix},
\end{equation}
where $G_{\mfix}$ denotes the cdf of $\mfix$ defined by
$G_{\mfix}(x) = \mfix((-\infty,x])$. Now if $h$ is monotonically increasing, we have $G_{h_{\sharp}\mfixAdd} = G_{\mfixAdd}\circ h^{-1}$, which implies compatibility.

Let $n\geq 1$ and $h\in \E$.
We first consider the case $h=S_{a}$ for some $a\in\RR^n$.
By \Cref{thm:brenier}, both $T_{\mfix}^{\mvariable}$ and $T_{\mfix}^{{S_a}_{\sharp}\mvariable}$ exist.
We now prove $T_{\mfix}^{{S_a}_{\sharp}\mvariable} = S_a \circ T_{\mfix}^{\mvariable}$, which shows the result for $h=S_{a}$.

Again, by \Cref{thm:brenier}, there exists a unique convex function $\varphi$ such that $T_{\mfix}^{\mvariable} = \nabla \varphi$. Then
\begin{equation*}
    \left( S_a \circ T_{\mfix}^{\mvariable} \right) (x) = \nabla \varphi (x) + a  = \nabla \left( \varphi(x) +\left<a,x\right>\right)= \nabla \psi(x),
\end{equation*}
where $\psi$ is also convex. 

Due to the general property 
\begin{equation}\label{eq:push_comp}
(\widetilde{T} \circ T)_{\sharp}\mfix = {\widetilde{T}}_{\sharp}(T_{\sharp}\mfix)
\end{equation}
for maps $T,\widetilde{T}$,
we have that
$S_a \circ T_{\mfix}^{\mvariable}$ pushes $\mfix$ to ${S_{a}}_{\sharp}\mvariable$.
Therefore \Cref{thm:brenier} implies that $S_a \circ T_{\mfix}^{\mvariable} = T_{\mfix}^{{S_{a}}_{\sharp}\mvariable}$.
%
%
%This proves the first part.
%
%

We now consider the case $h=R_c$ for some $c\in \RR$.
By \Cref{thm:brenier}, both $T_{\mfix}^{\mvariable}$ and $T_{\mfix}^{{R_c}_{\sharp}\mvariable}$ exist. 
We now prove that $T_{\mfix}^{{R_c}_{\sharp}\mvariable} = R_c \circ T_{\mfix}^{\mvariable}$, which implies the result for $h=R_c$. 

Again, by \Cref{thm:brenier}, there exists a unique convex function $\varphi$ such that $T_{\mfix}^{\mvariable} = \nabla \varphi$. Then
\begin{equation*}
    \left( R_c \circ T_{\mfix}^{\mvariable} \right) (x) = c\nabla \varphi (x)  = \nabla {c}\varphi(x)= \nabla \psi(x),
\end{equation*}
where $\psi$ is also convex.
Furthermore, by \cref{eq:push_comp}, $R_c \circ T_{\mfix}^{\mvariable}$ pushes $\mfix$ to ${R_c}_{\sharp}\mvariable$.
Therefore Theorem \ref{thm:brenier} implies
$T_{\mfix}^{{R_c}_{\sharp}\mvariable} = R_c \circ T_{\mfix}^{\mvariable}$.
%%%%%%%%%%%%%%%%%%%%%%%%%%%%%%
\end{proof}
%%%%%%%%%%%%%%%%%%%%%%%%%%%%%%%%%%%%

%%%%%%%%%%%%%%%%%%%%%%%%%%%%%%%%%%%%%%%%%%%%%%%%%%%%%%
\subsection{Proofs of \Cref{sec:linear_sep}}
\label{sec:ProofsLinearSep}
%%%%%%%%%%%%%%%%%%%%%%%%%%%%%%%%%%%%%%%%%%%%%%%%%%%%%%%
%
%%%%%%%%%%%%%%%%%%%%%%%%%%%%%%%%%%%%%%%%%%%%%%%%%%%%%%%%%%%%%%%%%%%%%%%%%%%%%%%%%%
We first establish an approximation result:
%
%
%%%%%%%%%%%%%%%%%%%%%%%%%%%%%%%%%%%%%%%%%%%%%%%%%%%%%%%
\begin{lemma}\label{lem:W2-CDT-approx}
%%%%%%%%%%%%%%%%%%%%%%%%%%%%%%%%%%%%%%%%%%%%%%%%%%%%%%%
Let $\mfix,\mfixAdd,\mvariable \in \BB(\RR^n)$, $\mfix,\mfixAdd \ll \lambda$, then we have
\begin{equation*}
W_2(\mfixAdd,\mvariable) \leq \|F_{\mfix}(\mfixAdd) - F_{\mfix}(\mvariable)\|_{\sigma} 
\leq W_2(\mfixAdd,\mvariable) + \|T_{\mfixAdd}^{\mvariable}-T^{\mvariable}_{\mfix}\circ T_{\mfixAdd}^{\mfix}\|_{\mfixAdd}.
\end{equation*}
We also have an upper bound by the triangle inequality
\begin{equation*}
\|F_{\mfix}(\mfixAdd) - F_{\mfix}(\mvariable)\|_{\sigma}\leq W_2(\mfixAdd,\mfix) + W_2(\mfix,\mvariable).
\end{equation*}
%%%%%%%%%%%%%%%%%%%%%%%%%%%%%%%%%%%%%%%%%%%%%%%%%%%%%%%
\end{lemma}
%%%%%%%%%%%%%%%%%%%%%%%%%%%%%%%%%%%%%%%%%%%%%%%%%%%%%%%%
%
%%%%%%%%%%%%%%%%%%%%%%%%%%%%%%%%%%%%%%%%%%%%%%%%%
\begin{proof}
%%%%%%%%%%%%%%%%%%%%%%%%%%%%%%%%%%%%%%%%%%%%%%%%%
By the change-of-variables formula we have
\begin{equation}\label{eq:change-variable}
    \|F_{\mfix}(\mfixAdd) - F_{\mfix}(\mvariable)\|_{\mfix}
    = \|T_{\mfix}^{\mfixAdd} - T_{\mfix}^{\mvariable} \|_{\mfix}
    = \|\id-T_{\mfix}^{\mvariable}\circ T^{\mfix}_{\mfixAdd}\|_{\mfixAdd}
\end{equation}
Since $T_{\mfix}^{\mvariable}\circ T_{\mfixAdd}^{\mfix}$ pushes $\mfixAdd$ to $\mvariable$, $W_2(\mfixAdd,\mvariable) \leq \|F_{\mfix}(\mfixAdd) - F_{\mfix}(\mvariable)\|_{\mfix}$ follows.

For the first upper bound on 
$\|F_{\mfix}(\mfixAdd) - F_{\mfix}(\mvariable)\|_{\mfix}$ note the following
\begin{equation*}
    \|\id-T_{\mfix}^{\mvariable}\circ T^{\mfix}_{\mfixAdd}\|_{\mfixAdd} 
    \leq \|\id-T_{\mfixAdd}^{\mvariable}\|_{\mfixAdd} 
    +
    \|T_{\mfixAdd}^{\mvariable} - T_{\mfix}^{\mvariable}\circ T^{\mfix}_{\mfixAdd}\|_{\mfixAdd}
    \leq W_2(\mfixAdd,\mvariable) + \|T_{\mfixAdd}^{\mvariable}-T_{\mfix}^{\mvariable}\circ T^{\mfix}_{\mfixAdd}\|_{\mfixAdd}.
\end{equation*}
The second upper bound by $W_2(\mfixAdd,\mfix) + W_2(\mfix,\mvariable)$ follows from the triangle inequality.
%%
%%%%%%%%%%%%%%%%%%%%%%%%%%%%%%%%%%%%%%%%%%%%%%%%%
\end{proof}
%%%%%%%%%%%%%%%%%%%%%%%%%%%%%%%%%%%%%%%%%%%%%%%%%%
%

\Cref{lem:W2-CDT-approx} shows that the error occurring in the LOT approximation of the Wasserstein distance is determined by the $L^2$-error between the map $T_{\mfix}^{\mfixAdd}\circ T^{\mfix}_{\mfixAdd}$ and the correct transport map $T_{\mfixAdd}^{\mvariable}$. This means that the LOT embedding replaces the transport $T_{\mfixAdd}^{\mvariable}$ by $T_{\mfix}^{\mfixAdd}\circ T^{\mfix}_{\mfixAdd}$ and computes the Wasserstein distance from this map.

\Cref{lem:W2-CDT-approx} shows that in case the relation
\begin{equation}\label{eq:transport_composition}
T_{\mfixAdd}^{\mvariable} = T_{\mfix}^{\mvariable}\circ T_{\mfixAdd}^{\mfix}
\end{equation}
is satisfied, the LOT embedding is an isometry. Also, if \cref{eq:transport_composition} is satisfied up to an error $\varepsilon>0$, then $\varepsilon$ is also the maximal error between the LOT embedding and the correct Wasserstein distance.

%%%%%%%%%%%%%%%%%%%%%%%%%%%%%%%%%%%%%%%
\begin{lemma}
\label{lem:compatible_isometry}
%%%%%%%%%%%%%%%%%%%%%%%%%%%%%%%%%%%%%%%%
Fix $\mfix,\mfixAdd \in \BB(\RR^n)$, $\mfix \ll \lambda$. If $F_{\mfix}$
is compatible with $\mu$-pushforwards of a set of functions $\HH \subseteq \FFm$ (see \cref{eq:compatible}) then for $h_1,h_2\in \HH$ we have
\begin{equation*}
  T_{{h_1}_{\sharp}\mfixAdd}^{{h_2}_{\sharp}\mfixAdd} = T^{{h_2}_{\sharp}\mfixAdd}_{\mfix} \circ T^{\mfix}_{{h_2}_{\sharp}\mfixAdd}.
\end{equation*}
%%
%%%%%%%%%%%%%%%%%%%%%%%%%%%%%%%%%%%%%%%%%%%%%%%
\end{lemma}
%%%%%%%%%%%%%%%%%%%%%%%%%%%%%%%%%%%%%%%
%%%%%%%%%%%%%%%%%%%%%%%%%%%%%%%%%%%%%%%%
\begin{proof}
%%%%%%%%%%%%%%%%%%%%%%%%%%%%%%%%%%%%%%%%%%
Denote by $\mvariable_1={h_1}_{\sharp}\mfixAdd, \mvariable_2={h_2}_{\sharp}\mfixAdd$.
Compatibility (\cref{eq:compatible}) implies
\begin{equation}
\label{eq:all_compatible_rewrite}
 T_{\mfix}^{\mvariable_j} = h_j \circ T_{\mfix}^{\mfixAdd}, \quad j=1,2.
\end{equation}
This implies
\begin{equation*}
    T_{\mfix}^{\mvariable_2} \circ T_{\mvariable_1}^{\mfix} = 
    h_2\circ T_{\mfix}^{\mfixAdd} \circ (h_1 \circ T_{\mfix}^{\mfixAdd})^{-1}
    =
    h_2 \circ h_1^{-1}.
\end{equation*}
Again by \cref{eq:compatible}, we obtain
\begin{equation*}
    T_{\mvariable_1}^{\mvariable_2} = 
    h_2 \circ T_{\mvariable_1}^{\mfixAdd} = 
    h_2 \circ (T^{\mvariable_1}_{\mfixAdd})^{-1}
    = h_2 \circ (h_1 \circ T_{\mfixAdd}^{\mfixAdd})^{-1}
    = h_2 \circ h_1^{-1}. \qedhere
\end{equation*}
%
%%%%%%%%%%%%%%%%%%%%%%%%%%%%%%%%%%%%%%%%%
\end{proof}
%%%%%%%%%%%%%%%%%%%%%%%%%%%%%%%%%%%%%%%%
%

%%%%%%%%%%%%%%%%%%%%%%%%%%%%%%%%%%%%%%%%%
\begin{proof}[Proof of \Cref{thm:CDT_perturb}]
%%%%%%%%%%%%%%%%%%%%%%%%%%%%%%%%%%%%%%%%%%%%%%%%

%%%%%%%%%%%%%%%%%%%%%%%%%%%%%%%%%%%%%%%%%%%%%%%%%%%%%%%%%%%%%%%%%%%%%%%%%
Since $g_1,g_2 \in \GG_{\mfixAdd,R,\varepsilon}$ there exist $h_1,h_2 \in \E_{\mfixAdd,R}$ such that $\|g_1-h_1\|_{\mfixAdd}<\varepsilon$ and $\|g_2-h_2\|_{\mfixAdd}<\varepsilon$. The triangle inequality implies
\begin{align}\label{eq:CDT_triangle}
  \|F_{\mfix}({g_1}_{\sharp}\mfixAdd)-& F_{\mfix}({g_2}_{\sharp}\mfixAdd)\|_{\mfix}
   \leq
  \|F_{\mfix}({g_1}_{\sharp}\mfixAdd)  -F_{\mfix}({h_1}_{\sharp}\mfixAdd)\|_{\mfix} \\ \nonumber
  & +
  \|F_{\mfix}({h_1}_{\sharp}\mfixAdd)-F_{\mfix}({h_2}_{\sharp}\mfixAdd)\|_{\mfix}
  +
  \|F_{\mfix}({h_2}_{\sharp}\mfixAdd)-F_{\mfix}({g_2}_{\sharp}\mfixAdd)\|_{\mfix}
\end{align}

%%%%%%%%%%%%%%%%%%%%%%%%%%%%%%%
\Cref{lem:isometry_compatible}, \cref{thm:pushforward_Lipschitz} and the triangle inequality imply
\begin{align}\nonumber
    \|F_{\mfix}({h_1}_{\sharp}\mfixAdd)& -F_{\mfix}({h_2}_{\sharp}\mfixAdd)\|_{\mfix} = W_2({h_1}_{\sharp}\mfixAdd,{h_2}_{\sharp}\mfixAdd) \\\nonumber
    &\leq 
    W_2({h_1}_{\sharp}\mfixAdd,{g_1}_{\sharp}\mfixAdd) + 
    W_2({g_1}_{\sharp}\mfixAdd,{g_2}_{\sharp}\mfixAdd) + 
    W_2({g_2}_{\sharp}\mfixAdd,{h_2}_{\sharp}\mfixAdd) \\\nonumber
    &\leq
    \|{h_1} - {g_1}\|_{\mfixAdd} + 
    W_2({g_1}_{\sharp}\mfixAdd,{g_2}_{\sharp}\mfixAdd) + 
    \|{g_2} - {h_2}\|_{\mfixAdd} \\ \label{eq:F_bound}
    & \leq 
    2\varepsilon + W_2({g_1}_{\sharp}\mfixAdd,{g_2}_{\sharp}\mfixAdd). 
\end{align}
% %%%%%%%%%%%%%%%%%%%%%%%%%%%%
% %%%%%%%%%%%%%%%%%%%%%%%
Now we distinguish the two cases of the theorem
\begin{enumerate}
    \item For this part, we use the following H\"older-$\frac{2}{15}$ regularity result by \cite{merigot20}: 
    \begin{equation}\label{eq:215bound}
        \|F_{\mfix}(\mvariable_1)-F_{\mfix}(\mvariable_2)\|_{\mfix}\leq C\,W_2(\mvariable_1,\mvariable_2)^{2/15},
    \end{equation}
    for $\mvariable_1,\mvariable_2\in \BB(\RR^n)$.
    For $i=1,2$ we get
    \begin{equation*}
        \|F_{\mfix}({g_i}_{\sharp}\mfixAdd) -F_{\mfix}({h_i}_{\sharp}\mfixAdd)\|_{\mfix} \leq C\,W_2({g_i}_{\sharp}\mfixAdd,{h_i}_{\sharp}\mfixAdd)^{2/15}\leq
        C\,\|g_i-h_i\|_{\mfixAdd}^{2/15}<C\,\varepsilon^{2/15}.
    \end{equation*}
    This, together with \eqref{eq:CDT_triangle} and \eqref{eq:F_bound} gives the overall bound 
    $$0\leq\|F_{\mfix}({g_1}_{\sharp}\mfixAdd)- F_{\mfix}({g_2}_{\sharp}\mfixAdd)\|_{\mfix}- W_2({g_1}_{\sharp}\mfixAdd,{g_2}_{\sharp}\mfixAdd)\leq 2C\,\varepsilon^{2/15} +2\varepsilon.$$
    \item With regularity assumptions on $\mfix,\mfixAdd$, \Cref{cor:push_CDT_bound_L2} implies that there exist constants $C_{\mfix,\mfixAdd,R}, \overline{C}_{\mfix,\mfixAdd,R}$ such that
    \begin{align}\label{eq:apply_Hoelder}
        \|F_{\mfix}({g_i}_{\sharp}\mfixAdd) -F_{\mfix}({h_i}_{\sharp}\mfixAdd)\|_{\mfix}
        &\leq C_{\mfix,\mfixAdd,R}\|g_i-h_i\|_{\mfixAdd} + \overline{C}_{\mfix,\mfixAdd,R}\|g_i-h_i\|^{1/2}_{\mfixAdd} \\ \nonumber
        &\leq C_{\mfix,\mfixAdd,R}\, \varepsilon + \overline{C}_{\mfix,\mfixAdd,R}\, \varepsilon^{1/2},
    \end{align}
for $i=1,2$. Note that the same constants can be used for $i=1$ and $i=2$ since $R$ bounds both $h_1$ and $h_2$. This, together with \eqref{eq:CDT_triangle} and \eqref{eq:F_bound} gives the overall bound
\begin{align*}
    0&\leq \|F_{\mfix}({g_1}_{\sharp}\mfixAdd)-F_{\mfix}({g_2}_{\sharp}\mfixAdd)\|_{\mfix}-W_2({g_1}_{\sharp}\mfixAdd,{g_2}_{\sharp}\mfixAdd) \\
    & \leq
    2(C_{\mfix,\mfixAdd,R}+1)\,\varepsilon + 2\overline{C}_{\mfix,\mfixAdd,R}\,\varepsilon^{1/2},
\end{align*}
\end{enumerate}
%%%%%%%%%%%%%%%%%%%%%%%%%%%%%%%%%%%%%%%%%%

%%

which concludes the proof.
%%%%%%%%%%%%%%%%%%%%%%%
\end{proof}
%%%%%%%%%%%%%%%%%%%%%%%
%

%

%
%
%%%%%%%%%%%%%%%%%%%%%%%%%
\begin{proof}[Proof of \Cref{cor:linear_sep_n1,cor:linear_sep_trans_scale}]
%%%%%%%%%%%%%%%%%%%%%%%%%
By \Cref{rem:dropCompatibility}, the compatibility condition \eqref{eq:sets_equal2} is satisfied. Thus we can apply \Cref{cor:linear_sep}.
%%%%%%%%%%%%%%%%%%%%%%%%%
\end{proof}
%%%%%%%%%%%%%%%%%%%%%%%%%
%
%

%%%%%%%%%%
\begin{proof}[Proof of \Cref{thm:almost_separability}]
%%%%%%%%%%
We show that $F_{\mfix}$ is $\delta$-compatible with both $\mfixAdd$- and $\mvariable$-orbits with respect to the action of $\GG$. Then the result follows from \Cref{cor:linear_sep_eps}. We note that the $\delta$ will be as in \Cref{remark:delta}.

Let $g \in \GG$ and $h\in \E_{\lambda,R}$ such that $\|g-h\|\leq \varepsilon$. Since $h\in \E_{\lambda,R}$ it is compatible with $\mfixAdd$-orbits.
First note that
\begin{align*}
    \|F_{\mfix}(g\star \mfixAdd)-g\star F_{\mfix}(\mfixAdd)\|_{\mfix}
     \leq 
    \|F_{\mfix}(g\star \mfixAdd) - F_{\mfix}(h\star \mfixAdd)\|_{\mfix}
    + \|h\star F_{\mfix}(\mfixAdd)-g\star F_{\mfix}(\mfixAdd)\|_{\mfix}.
\end{align*}
We further note that
\begin{equation*}
    \|h\star F_{\mfix}(\mfixAdd)-g\star F_{\mfix}(\mfixAdd)\|_{\mfix} = \|h\circ T_{\mfix}^{\mfixAdd}-g\circ T_{\mfix}^{\mfixAdd}\|_{\mfix}
    = \|h-g\|_{\mfixAdd} \leq \|f_{\mfixAdd}\|_{\infty}^{1/2}\,\varepsilon
\end{equation*}
To bound $\|F_{\mfix}(g\star \mfixAdd) - F_{\mfix}(h\star \mfixAdd)\|_{\mfix}$,
we distinguish the two cases as in the theorem:
\begin{enumerate}
    \item We use the H\"older bound \eqref{eq:215bound} and \eqref{thm:pushforward_Lipschitz}:
    \begin{align*}
        \|F_{\mfix}(g\star \mfixAdd) - F_{\mfix}(h\star \mfixAdd)\|_{\mfix} & \leq CW_2(g\star \mfixAdd,h\star \mfixAdd)^{2/15} \leq C\|g-h\|_{\mfixAdd}^{2/15} \\
        &\leq C\left(\|f_{\mfixAdd}\|_{\infty}^{1/2}\,\varepsilon \right)^{2/15}.
    \end{align*}
    Therefore, overall, $F_{\mfix}$ is $\delta$-compatible with $\delta = \|f_{\mfixAdd}\|_{\infty}^{1/2}\,\varepsilon+C\left(\|f_{\mfixAdd}\|_{\infty}^{1/2}\,\varepsilon \right)^{2/15}$.
    \item \Cref{cor:push_CDT_bound_L2} implies
    \begin{align*}
        & \|F_{\mfix}(g_{\sharp}\mfixAdd)-F_{\mfix}(h_{\sharp}\mfixAdd )\|_{\mfix}
           \\
        &\leq\left(\sqrt{\frac{4R}{{K_{\mfixAdd}^{\mfix}}}}+1\right)\,\|g-h\|_{\mfixAdd} 
         +
        \sqrt{4R\,\frac{W_2(\mfix,\mfixAdd)+R+\|\id\|_{\mfixAdd}}{K_{\mfixAdd}^{\mfix}}}\,
        \|g-h\|^{1/2}_{\mfixAdd} \\
        & \leq \left(\sqrt{\frac{4R}{{K_{\mfixAdd}^{\mfix}}}}+1\right)\left(\|f_{\mfixAdd}\|_{\infty}^{1/2}\,\varepsilon \right) 
        +\sqrt{4R\,\frac{W_2(\mfix,\mfixAdd)+R+\|\id\|_{\mfixAdd}}{K_{\mfixAdd}^{\mfix}}} \left(\|f_{\mfixAdd}\|_{\infty}^{1/2}\,\varepsilon \right)^{1/2}
    \end{align*}
Therefore, overall, $F_{\mfix}$ is $\delta$-compatible with 
$$\delta =\left(\sqrt{\frac{4R}{{K_{\mfixAdd}^{\mfix}}}}+2\right)\left(\|f_{\mfixAdd}\|_{\infty}^{1/2}\,\varepsilon \right) 
        +\sqrt{4R\,\frac{W_2(\mfix,\mfixAdd)+R+\|\id\|_{\mfixAdd}}{K_{\mfixAdd}^{\mfix}}} \left(\|f_{\mfixAdd}\|_{\infty}^{1/2}\,\varepsilon \right)^{1/2}.$$
\end{enumerate}
Similarly, it can be shown that $F_{\mfix}$ is $\delta$-compatible with $\mvariable$-orbits (now $\delta$ depending on $\mvariable$). Thus by taking the maximum between those $\delta$ values and multiplying by $6$ (distance conditions in \Cref{cor:linear_sep_eps}), all the assumptions of \Cref{cor:linear_sep_eps} are satisfied and linear separability follows.
%%%%%%%%%%
\end{proof}
%%%%%%%%%%%
%
%
%%%%%%%%%%%%%%%%%%%%%%%%%%%%%%%%%%%%%%%%%%%%%%%%%%%%%%
\subsection{A useful result in normed spaces}
%%%%%%%%%%%%%%%%%%%%%%%%%%%%%%%%%%%%%%%%%%%%%%%%%%%%%%
%%
In this section we derive a result on almost convex sets for general normed spaces. It states that if two almost convex sets are separated by a positive value, then their convex hull can also be separated.

This result is needed for the almost linear separability proof for perturbed shifts and scalings (\Cref{cor:linear_sep_eps,thm:almost_separability}).
%
%%%%%%%%%%%%%%%%%%%%%%%%%%%%%%%%%%%%%
\begin{definition}
\label{def:eps_convex}
%%%%%%%%%%%%%%%%%%%%%%%%%%%%%%%%%%%%
Let $(X,\|\cdot\|)$ be a normed space and let $\varepsilon>0$. $X$ is called $\varepsilon$-convex if for every $x_1,x_2\in X$ and $c\in[0,1]$ there exists $x\in X$ such that
\begin{equation*}
  \|(1-c)\,x_1+c\,x_2-x\|<\varepsilon.  
\end{equation*}
%%
%%%%%%%%%%%%%%%%%%%%%%%%%%%%%%%%%%%%%
\end{definition}
%%%%%%%%%%%%%%%%%%%%%%%%%%%%%%%%%%%%%
This definition states that for an $\varepsilon$-convex set $X$, $d(\conv(X),X)<\varepsilon$, where $\conv(X)$ denotes the convex hull of $X$ and $d$ is the distance between sets.
%
%
%%%%%%%%%%%%%%%%%%%%%%%%%%%%%%%%%%%%%%%
\begin{lemma}
\label{lem:distance_convex_hull}
%%%%%%%%%%%%%%%%%%%%%%%%%%%%%%%%%%%%%%%%
Let $(X,\|\cdot\|)$ be a normed space and let $\varepsilon>0$. Consider two $\varepsilon$-convex sets $A, B \subset X$. If $d(A,B)> 3\varepsilon$, then $d(\conv(A),\conv(B))>\epsilon$.
%%%%%%%%%%%%%%%%%%%%%%%%%%%%%%%%%%%%%%%%
\end{lemma}
%%%%%%%%%%%%%%%%%%%%%%%%%%%%%%%%%%%%%%%%
%
%%%%%%%%%%
\begin{proof}
%%%%%%%%%
Let $a\in A$ and $c_b \in \conv(B)$. Let $b\in B$ such that $\|c_b-b\|<\varepsilon$ Then
\begin{equation*}
    \|a-c_b\|\geq \|a-b\| - \|b-c_b\|> 3\varepsilon - \varepsilon = 2\varepsilon.
\end{equation*}
Therefore $d(A,\conv(B))>2\varepsilon$. Similarly one can prove that $d(B,\conv(A))>2\varepsilon$.

Now let $c_a \in \conv(A)$ and $c_b \in \conv(B)$ and choose $b\in B$ such that $\|c_b-b\|<\varepsilon$. Then we have
\begin{equation*}
    \|c_a-c_b\|\geq \|c_a-b\|-\|b-c_b\|> 2\varepsilon - \varepsilon = \varepsilon,
\end{equation*}
which implies that $d(\conv(A),\conv(B))>\varepsilon$.
%%%%%%%%%
\end{proof}
%%%%%%%%%%

%%%%%%%%%%%%%%%%%%%%%%%%%%%%%%%%%%%%%%%%%%%%%%%%%%%%%%%%%%%%%%%%%%%%%%%%%%%%%
% BIB
%%%%%%%%%%%%%%%%%%%%%%%%%%%%%%%%%%%%%%%%%%%%%%%%%%%%%%%%%%%%%%%%%%%%%%%%%%%%%

\bibliographystyle{siam}

\end{document}